\documentclass{article}

\usepackage{microtype}
\usepackage{graphicx}
\usepackage{subfigure}
\usepackage{booktabs} 

\usepackage{amsthm}
\usepackage{amsmath}
\usepackage{amssymb}
\usepackage{mathrsfs}
\usepackage{mathtools,amsfonts,stmaryrd}
\usepackage{bbm}

\newcommand{\R}{\mathbb{R}}
\newcommand{\Rd}{\mathbb{R}^d}
\newcommand{\Rdd}{\mathbb{R}^{d \times d}}
\newcommand{\PRd}{\mathscr{P}(\mathbb{R}^d)}
\newcommand{\PpRd}{\mathscr{P}_p(\mathbb{R}^d)}
\newcommand{\PdRd}{\mathscr{P}_2(\mathbb{R}^d)}
\newcommand{\ProjPWasserstein}[1]{\mathcal{P}_{#1}}
\newcommand{\ProjWasserstein}[1]{\mathcal{S}_{#1}}
\newcommand{\Wd}{\mathcal{W}}

\newcommand{\N}{\mathbb{N}}
\newcommand{\G}{\mathcal{G}}

\DeclareMathOperator*{\argmin}{arg\,min}
\DeclareMathOperator{\diag}{diag}
\DeclareMathOperator{\tr}{trace}
\DeclareMathOperator{\Id}{Id}
\newcommand{\dotp}[2]{\ensuremath{\langle #1 , #2\,\rangle}}
\def\ones{\mathbf{1}}
\def\zeros{\mathbf{0}}
\newtheorem{theorem}{Theorem}
\newtheorem{remark}{Remark}
\newtheorem{proposition}{Proposition}
\newtheorem{definition}{Definition}

\newtheorem{lemma}{Lemma}

\newcommand{\range}[1]{\llbracket#1\rrbracket}
\usepackage{hyperref}

\usepackage[accepted]{icml2019}

\icmltitlerunning{Subspace Robust Wasserstein Distances}

\begin{document}

\twocolumn[
\icmltitle{Subspace Robust Wasserstein Distances}

\begin{icmlauthorlist}
\icmlauthor{Fran\c{c}ois-Pierre Paty}{ensae}
\icmlauthor{Marco Cuturi}{google,ensae}
\end{icmlauthorlist}

\icmlaffiliation{ensae}{CREST-ENSAE, Palaiseau, France}
\icmlaffiliation{google}{Google Brain, Paris, France}

\icmlcorrespondingauthor{Fran\c{c}ois-Pierre Paty}{francois.pierre.paty@ensae.fr}

\icmlkeywords{Machine Learning, Optimal Transport, Wasserstein, ICML}

\vskip 0.3in
]



\printAffiliationsAndNotice{}

\begin{abstract}
Making sense of Wasserstein distances between discrete measures in high-dimensional settings remains a challenge. Recent work has advocated a two-step approach to improve robustness and facilitate the computation of optimal transport, using for instance projections on random real lines, or a preliminary quantization of the measures to reduce the size of their support. We propose in this work a ``max-min'' robust variant of the Wasserstein distance by considering the maximal possible distance that can be realized between two measures, assuming they can be projected orthogonally on a lower $k$-dimensional subspace. Alternatively, we show that the corresponding ``min-max'' OT problem has a tight convex relaxation which can be cast as that of finding an optimal transport plan with a low transportation cost, where the cost is alternatively defined as the sum of the $k$ largest eigenvalues of the second order moment matrix of the displacements (or matchings) corresponding to that plan (the usual OT definition only considers the trace of that matrix). We show that both quantities inherit several favorable properties from the OT geometry. We propose two algorithms to compute the latter formulation using entropic regularization, and illustrate the interest of this approach empirically.
\end{abstract}
\section{Introduction}
The optimal transport (OT) toolbox~\citep{Villani09} is gaining popularity in machine learning, with several applications to data science outlined in the recent review paper~\citep{COTFNT}. When using OT on high-dimensional data, practitioners are often confronted to the intrinsic instability of OT with respect to input measures. A well known result states for instance that the sample complexity of Wasserstein distances can grow exponentially in dimension~\citep{dudley1969speed,fournier2015rate}, which means that an irrealistic amount of samples from two continuous measures is needed to approximate faithfully the true distance between them. This result can be mitigated when data lives on lower dimensional manifolds as shown in~\citep{weed2017sharp}, but sample complexity bounds remain pessimistic even in that case. From a computational point of view, that problem can be interpreted as that of a lack of robustness and instability of OT metrics with respect to their inputs. This fact was already a common concern of the community when these tools were first adopted, as can be seen in the use of $\ell_1$ costs~\citep{TreeEMD2007} or in the common practice of thresholding cost matrices~\citep{Pele-iccv2009}.

\begin{figure*}[!ht]
	\includegraphics[width=\textwidth]{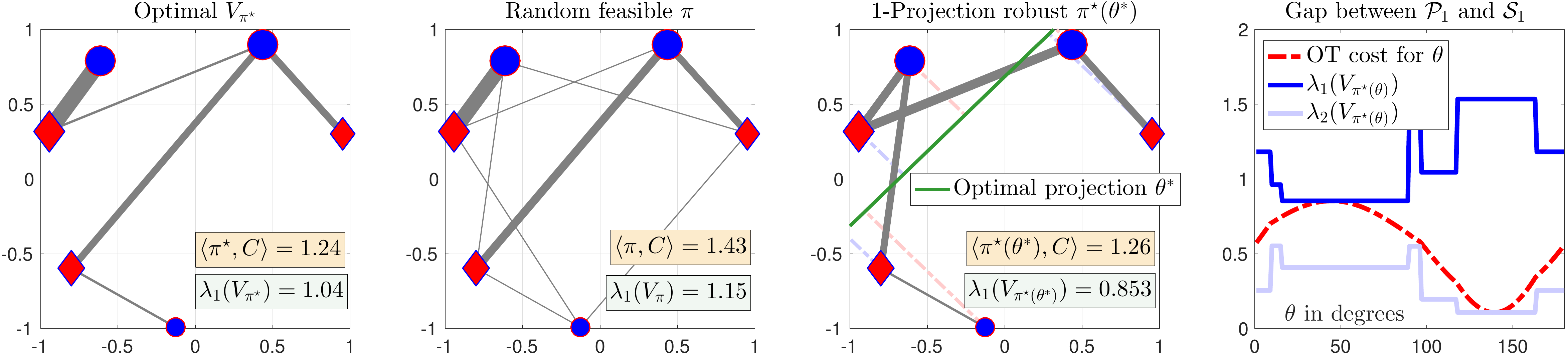}
	\caption{We consider two discrete measures (red and blue dots) on the plane. The left-most plot shows the optimal transport between these points, in which the width of the segment is proportional to the mass transported between two locations. The total cost is displayed in the lower right part of the plot as $\dotp{\pi^\star}{C}$, where $C$ is the pairwise squared-Euclidean distance matrix. The largest eigenvalue of the corresponding second order moment matrix $V_{\pi^\star}$ of displacements, see~\eqref{eq:vpi}, is given below. As can be expected and seen in the second plot, choosing a random transportation plan yields a higher cost. The third plot displays the most robust projection direction (green line), that upon which the OT cost of these point clouds is largest once projected. The maximal eigenvalue of the second order moment matrix (still in dimension 2) is smaller than that obtained with the initial OT plan. Finally, we plot as a function of the angle $\theta$ between $(0,180)$ the OT cost (which, in agreement with the third plot, is largest for the angle corresponding to the green line of the third plot) as well as the corresponding maximal eigenvalue of the second order moment of the optimal plan corresponding to \emph{each} of these angles $\theta$. The maximum of the red curve, as well as the minimum reached by the dark blue one, correspond respectively to the values of the projection $\ProjPWasserstein{k}$ and subspace $\ProjWasserstein{k}$ robust Wasserstein distances described in \S3. They happen to coincide in this example, but one may find examples in which they do not, as can be seen in Figure~\ref{fig:lafigurenulle} (supplementary material). The smallest eigenvalue is given for illustrative purposes only.}\label{fig:lasuperfigure}
\end{figure*}

\paragraph{Regularization} The idea to trade off a little optimality in exchange for more regularity is by now considered a crucial ingredient to make OT work in data sciences. A line of work initiated in~\citep{CuturiSinkhorn} advocates adding an entropic penalty to the original OT problem, which results in faster and differentiable quantities, as well as improved sample complexity bounds~\citep{genevay19}. Following this, other regularizations~\citep{dessein2016regularized}, notably quadratic~\citep{blondel2018smooth}, have also been investigated. Sticking to an entropic regularization, one can also interpret the recent proposal by~\citet{altschuler2018massively} to approximate Gaussian kernel matrices appearing in the regularized OT problem with Nystr\"om-type factorizations (or exact features using a Taylor expansion~\citep{cotter2011explicit} as in~\citep{altschuler2018approximating}), as robust approaches that are willing to tradeoff yet a little more cost optimality in exchange for faster Sinkhorn iterations. In a different line of work, quantizing first the measures to be compared before carrying out OT on the resulting distriutions of centroids is a fruitful alternative~\citep{CanasRosasco} which has been recently revisited in~\citep{weed19}. Another approach exploits the fact that the OT problem between two distributions on the real line boils down to the direct comparison of their generalized quantile functions~\citep[\S2]{SantambrogioBook}. Computing quantile functions only requires sorting values, with a mere log-linear complexity. The \emph{sliced} approximation of OT~\cite{rabin-ssvm-11} consists in projecting two probability distributions on a given line, compute the optimal transport cost between these projected values, and then repeat this procedure to average these distances over several random lines. This approach can be used to define kernels~\citep{kolouri2016sliced}, compute barycenters~\citep{2013-Bonneel-barycenter} but also to train generative models~\citep{kolouri2018sliced,Deshpande_2018_CVPR}. Beyond its practical applicability, this approach is based on a perhaps surprising point-of-view: OT on the real line may be sufficient to extract geometric information from two high-dimensional distributions. Our work builds upon this idea, and more candidly asks what can be extracted from a little more than a real line, namely a subspace of dimension $k\geq 2$. Rather than project two measures on several lines, we consider in this paper projecting them on a $k$-dimensional subspace that maximizes their transport cost. This results in optimizing the Wasserstein distance over the ground metric, which was already considered for supervised learning~\citep{cuturi2014ground,flamary2018wasserstein}.


\paragraph{Contributions} This optimal projection translates into a ``max-min'' robust OT problem with desirable features. Although that formulation cannot be solved with convex solvers, we show that the corresponding ``min-max'' problem admits on the contrary a tight convex relaxation and also has an intuitive interpretation. To see that, one can first notice that the usual 2-Wasserstein distance can be described as the minimization of the \emph{trace} of the second order moment matrix of the displacements associated with a transport plan. We show that computing a maximally discriminating optimal $k$ dimensional subspace in this ``min-max'' formulation can be carried out by minimizing the sum of the $k$ largest eigenvalues (instead of the entire trace) of that second order moment matrix. A simple example summarizing the link between these two ``min-max'' and ``max-min'' quantities is given in Figure~\ref{fig:lasuperfigure}. That figure considers a toy example where points in dimension $d=2$ are projected on lines $k=1$, our idea is designed to work for larger $k$ and $d$, as shown in~\S\ref{sec:exp}.

\paragraph{Paper structure} We start this paper with background material on Wasserstein distances in \S\ref{sec:background} and present an alternative formulation for the 2-Wasserstein distance using the second order moment matrix of displacements described in a transport plan. We define in \S\ref{sec:definitions} our ``max-min'' and ``min-max'' formulations for, respectively, projection (PRW) and subspace (SRW) robust Wasserstein distances. We study the geodesic structure induced by the SRW distance on the space of probability measures in \S\ref{sec:geometry}, as well as its dependence on the dimension parameter $k$. We provide computational tools to evaluate SRW using entropic regularization in \S\ref{sec:algos}. We conclude the paper with experiments in \S\ref{sec:exp} to validate and illustrate our claims, on both simulated and real datasets.

\vskip1cm

\section{Background on Optimal Transport}\label{sec:background}

For $d \in \N$, we write $\range{d} = \{1, ..., d\}$. Let $\PRd$ be the set of Borel probability measures in $\mathbb{R}^d$, and let 
\[
    \PdRd = \left\{ \mu \in \PRd \,\bigg|\, \int \|x\|^2 \,d\mu(x) < \infty \right\}.
\]

\paragraph{Monge and Kantorovich Formulations of OT}

For $\mu, \nu \in \PRd$, we write $\Pi(\mu, \nu)$ for the set of couplings
\begin{multline*}
        \Pi(\mu, \nu) = \{ \pi \in \mathscr{P}(\Rd \times \Rd) \textrm{ s.t.} \,\forall A, B\subset\Rd \text{ Borel},\\
        \pi(A\times \Rd)=\mu(A), \pi(\Rd \times B)=\nu(B) \},
\end{multline*}
and their $2$-Wasserstein distance is defined as
\[
    \mathcal{W}_2(\mu, \nu) := \left( \inf_{\pi \in \Pi(\mu, \nu)} \int \|x-y\|^2 \,d\pi(x,y) \right)^{1/2}.
\]
Because we only consider quadratic costs in the remainder of this paper, we drop the subscript $2$ in our notation and will only use $\Wd$ to denote the $2$-Wasserstein distance. For Borel $\mathcal{X}, \mathcal{Y} \subset \Rd$, Borel $T : \mathcal{X} \to \mathcal{Y}$ and $\mu \in \mathscr{P}(\mathcal{X})$, we denote by $T_\#\mu \in \mathscr{P}(\mathcal{Y})$ the push-forward of $\mu$ by $T$, \emph{i.e.} the measure such that for any Borel set $A \subset \mathcal{Y}$,
\[
	T_\#\mu(A) = \mu\left( T^{-1}(A) \right).
\]
The~\citet{Monge1781} formulation of optimal transport is, when this minimization is feasible, equivalent to that of Kantorovich, namely 
\[
    \Wd(\mu, \nu) = \left( \inf_{T: T_\#\mu=\nu} \int \|x-T(x)\|^2 \,d\mu(x) \right)^{1/2}.
\]

\paragraph{$\Wd$ as Trace-minimization}
For any coupling $\pi$, we define the $d\times d$ second order displacement matrix
\begin{equation}\label{eq:vpi}
	V_\pi := \int (x-y)(x-y)^T d\pi(x,y).
\end{equation}
Notice that when a coupling $\pi$ corresponds to a Monge map, namely $\pi=(\Id,T)_\#\mu$, then one can interpret even more naturally $V_\pi$ as the second order moment of all displacement $(x-T(x))(x-T(x))^T$ weighted by $\mu$. With this convention, we remark that the total cost of a coupling $\pi$ is equal to the trace of $V_\pi$, using the simple identity $\tr(x-y)(x-y)^T=\|x-y\|^2$ and the linearity of the integral sum.  Computing the $\Wd$ distance can therefore be interpreted as minimizing the trace of $V_\pi$. This simple observation will play an important role in the next section, and more specifically the study of $\lambda_l(V_\pi)$, the $l$-th \emph{largest} eigenvalue of $V_\pi$.
\section{Subspace Robust Wasserstein Distances}\label{sec:definitions}
With the conventions and notations provided in \S2, we consider here different robust formulations of the Wasserstein distance. Consider first for $k \in \range{d}$, the Grassmannian of $k$-dimensional subspaces of $\Rd$ : 
\[
	\G_k = \left\{ E \subset \Rd \,|\, \textrm{dim}(E) = k \right\}.
\]
For $E\in \G_k$, we note $P_E$ the orthogonal projector onto $E$. Given two measures $\mu,\nu \in \PdRd$, a first attempt at computing a robust version of $\Wd(\mu,\nu)$ is to consider the worst possible OT cost over all possible low dimensional projections:
\begin{definition}
For $k \in \range{d}$, the $k$-dimensional projection robust $2$-Wasserstein (PRW) distance between $\mu$ and $\nu$ is
\[
	\ProjPWasserstein{k}(\mu, \nu) = \sup_{E \in \G_k} \Wd \left( {P_E}_\#\mu , {P_E}_\#\nu \right).
\]	
\end{definition}
As we show in the supplementary material, this quantity is well posed and itself worthy of interest, yet difficult to compute. In this paper, we focus our attention on the corresponding ``min-max'' problem, to define the $k$-dimensional subspace robust $2$-Wasserstein (SRW) distance:
\
\begin{definition}
For $k \in \range{d}$, the $k$-dimensional subspace robust $2$-Wasserstein distance between $\mu$ and $\nu$ is
\[
	\begin{split}\label{eq:minPmaxE}
		 \ProjWasserstein{k}(\mu, \nu) &= \adjustlimits\inf_{\pi \in \Pi(\mu , \nu)} \sup_{E \in \G_k} \left[ \int \|P_E(x-y)\|^2 d\pi(x,y) \right]^{1/2} 
	\end{split}
\]
\end{definition}

\begin{remark}
Both quantities $\ProjWasserstein{k}$ and $\ProjPWasserstein{k}$ can be interpreted as robust variants of the $\Wd$ distance. By a simple application of weak duality we have that $\ProjPWasserstein{k}(\mu, \nu) \leq  \ProjWasserstein{k}(\mu, \nu)$.
\end{remark}

\begin{lemma} \label{lemma:infsup_is_minmax} Optimal solutions for $\ProjWasserstein{k}$ exist, \emph{i.e.}
\[
	\ProjWasserstein{k}(\mu, \nu) = \adjustlimits\min_{\pi \in \Pi(\mu , \nu)} \max_{E \in \G_k} \left[ \int \|P_E(x-y)\|^2 d\pi(x,y) \right]^{1/2}
\]
\end{lemma}

We show next that the SRW variant $\ProjWasserstein{k}$ can be elegantly reformulated as a function of the eigendecomposition of the displacement second-order moment matrix $V_\pi$~\eqref{eq:vpi}: 
\begin{lemma} \label{lemma:dual_formulation}
For $k \in \range{d}$ and $\mu, \nu \in \PdRd$, one has 
\[
	\begin{split}
		\ProjWasserstein{k}^2(\mu, \nu) &= \min_{\pi \in \Pi(\mu, \nu)}  \max_{\substack{U \in \R^{k \times d}\\UU^T = I_k}} \int \|Ux - Uy\|^2 \,d\pi(x,y) \\
		&= \min_{\pi \in \Pi(\mu, \nu)} \sum_{l=1}^k \lambda_l(V_\pi). 
	\end{split}
\]
\end{lemma}

This characterization as a sum of eigenvalues will be crucial to study theoretical properties of $\ProjWasserstein{k}$. Subspace robust Wasserstein distances can in fact be interpreted as a convex relaxation of projection robust Wasserstein distances: they can be computed as the maximum of a concave function over a convex set, which will make computations tractable.
\begin{theorem} \label{theo:convex_relaxation}
For $k \in \range{d}$ and $\mu, \nu \in \PdRd$,
\begin{align}
	\ProjWasserstein{k}^2(\mu, \nu) &= \min_{\pi \in \Pi(\mu, \nu)} \max_{\substack{0 \preceq \Omega \preceq I\\\tr(\Omega) = k}}  \int d_\Omega^2 \,d\pi \label{eqn:minmaxOmega} \\
	&= \max_{\substack{0 \preceq \Omega \preceq I\\\tr(\Omega) = k}} \min_{\pi \in \Pi(\mu, \nu)} \int d_\Omega^2 \,d\pi \label{eqn:maxOmegamin} \\
	&= \max_{\substack{0 \preceq \Omega \preceq I\\\tr(\Omega) = k}} \Wd^2 \left( {\Omega^{1/2}}_\#\mu , {\Omega^{1/2}}_\#\nu \right) \label{eqn:maxOmegaWass}
\end{align}
where $d_\Omega$ stands for the Mahalanobis distance
\[
	d_\Omega^2(x,y) = (x-y)^T \Omega (x-y).
\]
\end{theorem}

We can now prove that \emph{both} PRW and SRW variants are, indeed, distances over $\PdRd$.
\begin{proposition}
For $k \in \range{d}$, both $\ProjPWasserstein{k}$ and $\ProjWasserstein{k}$ are distances over $\PdRd$.
\end{proposition}

\emph{Proof.} Symmetry is clear for both objects, and for $\mu \in \PdRd$, $\ProjWasserstein{k}(\mu, \mu) = \ProjPWasserstein{k}(\mu, \mu) = 0$. Let $\mu, \nu \in \PdRd$ such that $\ProjWasserstein{k}(\mu, \nu) = 0$. Then $\ProjPWasserstein{k}(\mu, \nu) = 0$ and for any $E \in \G_k$, $\Wd({P_E}_\#\mu , {P_E}_\#\nu) = 0$, \emph{i.e.} ${P_E}_\#\mu = {P_E}_\#\nu$. Lemma~\ref{lemma:renyi-cramer-wold} (in the supplementary material) then shows that $\mu = \nu$.
For the triangle inequalities, let $\mu_0, \mu_1, \mu_2 \in \PdRd$. Let $\Omega_\star \in \left\{0 \preceq \Omega \preceq I ,\, \tr(\Omega) = k \right\}$ be optimal between $\mu_0$ and $\mu_2$. Using the triangle inequalities for the Wasserstein distance,
\[
	\begin{split}
		&\ProjWasserstein{k}(\mu_0, \mu_2) = \Wd\left[{\Omega_\star^{1/2}}_\#\mu_0, {\Omega_\star^{1/2}}_\#\mu_2\right] \\
     		&\leq \Wd\left[{\Omega_\star^{1/2}}_\#\mu_0, {\Omega_\star^{1/2}}_\#\mu_1\right] + \Wd\left[{\Omega_\star^{1/2}}_\#\mu_1, {\Omega_\star^{1/2}}_\#\mu_2\right] \\
     		&\leq \sup_{\substack{0 \preceq \Omega \preceq I\\\tr(\Omega) = k}} \Wd\left[{\Omega^{1/2}}_\#\mu_0, {\Omega^{1/2}}_\#\mu_1\right]  \\
		&+ \sup_{\substack{0 \preceq \Omega \preceq I\\\tr(\Omega) = k}} \Wd\left[{\Omega^{1/2}}_\#\mu_1, {\Omega^{1/2}}_\#\mu_2\right] \\
     		&= \ProjWasserstein{k}(\mu_0, \mu_1) + \ProjWasserstein{k}(\mu_1, \mu_2).
	\end{split}
\]
The same argument, used this time with projections, yields the triangle inequalities for $\ProjPWasserstein{k}$. \quad  $\square$
\section{Geometry of Subspace Robust Distances}\label{sec:geometry}
We prove in this section that SRW distances share several fundamental geometric properties with the Wasserstein distance. The first one states that distances between Diracs match the ground metric:

\begin{lemma} \label{lemma:tight_bounds}
For $x,y \in \Rd$ and $k \in \range{d}$,
\[
	\ProjWasserstein{k}(\delta_x,\delta_y)=\|x-y\|.
\]
\end{lemma}

\paragraph{Metric Equivalence.}  Subspace robust Wasserstein distances $\ProjWasserstein{k}$ are equivalent to the Wasserstein distance $\Wd$:
\begin{proposition} \label{prop:equivalent_metrics}
For $k \in \range{d}$, $\ProjWasserstein{k}$ is equivalent to $\Wd$. More precisely, for $\mu, \nu \in \PdRd$,
\[
    \sqrt{\frac{k}{d}} \Wd(\mu, \nu) \leq \ProjWasserstein{k}(\mu, \nu) \leq \Wd(\mu, \nu).
\]
Moreover, the constants are tight since
\[
\begin{array}{lcl}
    \ProjWasserstein{k}(\delta_x, \delta_y) &=& \Wd(\delta_x, \delta_y) \\
    \ProjWasserstein{k}(\delta_0, \sigma) &=& \sqrt{\frac{k}{d}} \Wd(\delta_0, \sigma)
\end{array}
\]
where $\delta_x, \delta_y, \delta_0$ are Dirac masses at points $x,y,0 \in \Rd$ and $\sigma$ is the uniform probability distribution over the centered unit sphere in $\Rd$.
\end{proposition}

\paragraph{Dependence on the dimension.} We fix $\mu, \nu \in \PdRd$ and we ask the following question : how does $\ProjWasserstein{k}(\mu, \nu)$ depend on the dimension $k \in \range{d}$ ? The following lemma gives a result in terms of eigenvalues of $V_{\pi_k}$, where $\pi_k \in \Pi(\mu,\nu)$ is optimal for some dimension $k$, then we translate in Proposition~\ref{prop:dependence_dimension} this result in terms of $\ProjWasserstein{k}$.

\begin{lemma} \label{lemma:increasing_concave_in_K}
Let $\mu, \nu \in \PdRd$. For any $k \in \range{d-1}$,
\[
	\begin{split}
        		\lambda_{k+1}(V_{\pi_{k+1}})& \leq \ProjWasserstein{k+1}^2(\mu, \nu) - \ProjWasserstein{k}^2(\mu, \nu) \\
    		& \leq \lambda_{k+1}(V_{\pi_k})
	\end{split}
\]
where for $L \in \range{d}$, $\displaystyle \pi_L \in \argmin_{\pi \in \Pi(\mu, \nu)} \sum_{l=1}^L \lambda_l(V_\pi)$.
\end{lemma}

\begin{proposition} \label{prop:dependence_dimension}
Let $\mu, \nu \in \PdRd$. The sequence $k \mapsto \ProjWasserstein{k}^2(\mu,\nu)$ is increasing and concave. In particular, for $k \in \range{d-1}$,      
\[
	\ProjWasserstein{k+1}^2(\mu,\nu) - \ProjWasserstein{k}^2(\mu,\nu) 
      	\geq \frac{\Wd^2(\mu,\nu) - \ProjWasserstein{k}^2(\mu,\nu)}{d-k}.
\]
Moreover, for any $k \in \range{d-1}$,
\[
    \ProjWasserstein{k}(\mu,\nu) \leq \ProjWasserstein{k+1}(\mu,\nu) \leq \sqrt{\frac{k+1}{k}} \ProjWasserstein{k}(\mu,\nu).
\]
\end{proposition}

\paragraph{Geodesics}
We have shown in Proposition~\ref{prop:equivalent_metrics} that for any $k \in \range{d}$, $\left(\PdRd, \ProjWasserstein{k}\right)$ is a metric space with the same topology as that of the Wasserstein space $\left(\PdRd, \Wd \right)$. We conclude this section by showing that $\left(\PdRd, \ProjWasserstein{k}\right)$ is in fact a geodesic length space, and exhibits explicit constant speed geodesics. This can be used to interpolate between measures in $\ProjWasserstein{k}$ sense.

\begin{proposition} \label{prop:geodesic_space}
Let $\mu, \nu \in \PdRd$ and $k \in \range{d}$. Take 
\[
	\pi^* \in \argmin_{\pi \in \Pi(\mu,\nu)} \sum_{l=1}^k \lambda_l(V_\pi)
\]
and let $f_t(x,y) = (1-t)x + ty$. Then the curve 
\[
	t \mapsto \mu_t := {f_t}_\#\pi^*
\]
is a constant speed geodesic in $\left( \PdRd, \ProjWasserstein{k} \right)$ connecting $\mu$ and $\nu$. Consequently, $\left( \PdRd, \ProjWasserstein{k} \right)$ is a geodesic space.
\end{proposition}
\emph{Proof.}
We first show that for any $s,t \in [0,1]$,
\[
	\ProjWasserstein{k}(\mu_s, \mu_t) = |t-s| \ProjWasserstein{k}(\mu, \nu)
\]
by computing the cost of the transport plan $\pi(s,t) = (f_s, f_t)_\#\pi^* \in \Pi(\mu_s, \mu_t)$ and using the triangular inequality. Then the curve $(\mu_t)$ has constant speed 
\[
	|\mu_t'| = \lim_{\epsilon \to 0} \frac{\ProjWasserstein{k}(\mu_{t+\epsilon}, \mu_t)}{|\epsilon|} = \ProjWasserstein{k}(\mu, \nu),
\]
and the length of the curve $(\mu_t)$ is 
\[
	\begin{split}
		&\sup\left\{ \sum_{i=0}^{n-1} \ProjWasserstein{k}(\mu_{t_i}, \mu_{t_{i+1}}) \,\bigg|\, 
		\begin{array}{c}
			n \geq 1\\
			0 = t_0 < ... < t_n = 1
		\end{array}
		\right\}\\
		&= \ProjWasserstein{k}(\mu, \nu),
	\end{split}
\]
\emph{i.e.} $(\mu_t)$ is a geodesic connecting $\mu$ and $\nu$.\quad$\square$

\section{Computation}\label{sec:algos} 

We provide in this section algorithms to compute the saddle point solution of $\ProjWasserstein{k}$. $\mu, \nu$ are now discrete with respectively $n$ and $m$ points and weights $a$ and $b$ : $\mu := \sum_{i=1}^n a_i \delta_{x_i}$ and $\nu := \sum_{j=1}^m b_i \delta_{y_j}$. For $k \in \range{d}$, three different objects are of interest: \emph{(i)} the value of $\ProjWasserstein{k}(\mu, \nu)$, \emph{(ii)} an optimal subspace $E^*$ obtained through the relaxation for SRW, \emph{(iii)} an optimal transport plan solving SRW.
A subspace can be used for dimensionality reduction, whereas an optimal transport plan can be used to compute a geodesic, \emph{i.e.} to interpolate between $\mu$ and $\nu$.

\subsection{Computational challenges to approximate $\ProjWasserstein{k}$}

We observe that solving $\min_{\pi \in \Pi(\mu,\nu)} \sum_{l=1}^k \lambda_l(V_\pi)$ is challenging. Considering a direct projection onto the transportation polytope
\[
	\Pi(\mu, \nu) = \left\{ \pi \in \R^{n \times m} \,|\, \pi\ones_m = a, \, \pi^T\ones_n = b\right\}
\]
would result in a costly quadratic network flow problem. The Frank-Wolfe algorithm, which does not require such projections, cannot be used directly because the application $\pi \mapsto \sum_{l=1}^k \lambda_l(V_\pi)$ is not smooth.

On the other hand, thanks to Theorem~\ref{theo:convex_relaxation}, solving the maximization problem is easier. Indeed, we can project onto the set of constraints $\mathcal{R} = \{ \Omega \in \Rdd \,|\, 0 \preceq \Omega \preceq I \,;\, \tr(\Omega) = k \}$ using Dykstra's projection algorithm~\cite{boyle1986method}. In this case, we will only get the value of $\ProjWasserstein{k}(\mu, \nu)$ and an optimal subspace, but not necessarily the actual optimal transport plan due to the lack of uniqueness for OT plans in general.

\paragraph{Smoothing} It is well known that saddle points are hard to compute for a bilinear objective~\cite{hammond1984solving}. Computations are greatly facilitated by adding smoothness, which allows the use of saddle point Frank-Wolfe algorithms~\cite{gidel2016frank}. Out of the two problems, the maximization problem is seemingly easier. Indeed, we can leverage the framework of regularized OT~\cite{CuturiSinkhorn} to output, using Sinkhorn's algorithm, a unique optimal transport plan $\pi^\star$ at each inner loop of the maximization. To save time, we remark that initial iterations can be solved with a low accuracy by limiting the number of iterations, and benefit from warm starts, using the scalings computed at the previous iteration, see~\citep[\S4]{COTFNT}. 

\subsection{Projected Supergradient Method for SRW}
In order to compute SRW and an optimal subspace, we can solve equation~\eqref{eqn:maxOmegamin} by maximizing the concave function
\[
	f : \Omega \mapsto \min_{\pi \in \Pi(\mu, \nu)} \sum_{i,j} d_\Omega^2(x_i, y_j) \,\pi_{i,j} = \min_{\pi \in \Pi(\mu, \nu)} \langle \Omega \,|\, V_\pi \rangle
\]
over the convex set $\mathcal{R}$. Since $f$ is not differentiable, but only superdifferentiable, we can only use a projected supergradient method. This algorithm is outlined in Algorithm~\ref{alg:projected_supergradient}. Note that by Danskin's theorem, for any $\Omega \in \mathcal{R}$,
\[
	\partial f(\Omega) = \textrm{Conv}\left\{V_{\pi^*} \,\bigg|\, \pi^* \in \argmin_{\pi \in \Pi(\mu, \nu)} \langle \Omega \,|\, V_\pi \rangle \right\}.
\]

\begin{algorithm}[tb]
	\caption{Projected supergradient method for SRW}
	\label{alg:projected_supergradient}
	\begin{algorithmic}
   		\STATE {\bfseries Input:} Measures $(x_i, a_i)$ and $(y_j, b_j)$, dimension $k$, learning rate $\tau_0$
		\STATE $\pi \leftarrow$ OT$((x,a),(y,b), cost=\|\cdot\|^2)$
		\STATE $U \leftarrow$ top $k$ eigenvectors of $V_\pi$
		\STATE Initialize $\Omega = UU^T \in \Rdd$
   		\FOR{$t=0$ {\bfseries to} max\_iter}
		\STATE $\pi \leftarrow$ OT$((x,a), (y,b), cost=d_\Omega^2)$
		\STATE $\tau = \tau_0/(t+1)$
		\STATE $\Omega \leftarrow$ Proj$_{\mathcal{R}}\left[ \Omega + \tau V_\pi \right]$
  		\ENDFOR
		\STATE{\bfseries Output:} $\Omega, \langle \Omega \,|\, V_\pi \rangle$
	\end{algorithmic}
\end{algorithm}

\subsection{Frank-Wolfe using Entropy Regularization}
Entropy-regularized optimal transport can be used to compute a unique optimal plan given a subspace. Let $\gamma > 0$ be the regularization strength. In this case, we want to maximize the concave function
\[
	f_\gamma : \Omega \mapsto \min_{\pi \in \Pi(\mu, \nu)} \langle \Omega \,|\, V_\pi \rangle + \gamma \sum_{i,j} \pi_{i,j} \left[ \log(\pi_{i,j}) - 1 \right]
\]
over the convex set $\mathcal{R}$. Since for all $\Omega \in \mathcal{R}$, there is a unique $\pi^*$ minimizing $\pi \mapsto \langle \Omega \,|\, V_\pi \rangle + \gamma \sum_{i,j} \pi_{i,j} \left[ \log(\pi_{i,j}) - 1 \right]$, $f_\gamma$ is differentiable. Instead of running a projected gradient ascent on $\Omega \in \mathcal{R}$, we propose to use the Frank-Wolfe algorithm when the regularization strength is positive. Indeed, there is no need to tune a learning rate in Frank-Wolfe, making it easier to use. We only need to compute, for fixed $\pi \in \Pi(\mu, \nu)$, the maximum over $\mathcal{R}$ of $\Omega \mapsto \langle \Omega \,|\, V_\pi \rangle$: 
\begin{lemma} \label{lemma:optimalOmega_fixedpi}
For $\pi \in \Pi(\mu, \nu)$, compute the eigendecomposition of $V_\pi = U \diag\left( \lambda_1, ..., \lambda_d \right) U^T$
with $\lambda_1 \geq ... \geq \lambda_d$. Then for $k \in \range{d}$, $\widehat\Omega = U \diag\left([\ones_k, \zeros_{d-k}]\right) U^T$ solves
\[
	\max_{\substack{0 \preceq \Omega \preceq I\\\tr(\Omega) = k}}  \int d_\Omega^2 \,d\pi.
\]
\end{lemma}
This algorithm is outlined in algorithm~\ref{alg:FW}.

\begin{algorithm}[tb]
	\caption{Frank-Wolfe algorithm for regularized SRW}
	\label{alg:FW}
	\begin{algorithmic}
   		\STATE {\bfseries Input:} Measures $(x_i, a_i)$ and $(y_j, b_j)$, dimension $k$, regularization strength $\gamma > 0$, precision $\epsilon > 0$
		\STATE $\pi \leftarrow$ reg\_OT$((x,a),(y,b), reg=\gamma, cost=\|\cdot\|^2)$
		\STATE $U \leftarrow$ top $k$ eigenvectors of $V_\pi$
		\STATE Initialize $\Omega = UU^T \in \Rdd$
   		\FOR{$t=0$ {\bfseries to} max\_iter}
		\STATE $\pi \leftarrow$ reg\_OT$((x,a), (y,b), reg=\gamma, cost=d_\Omega^2)$
		\STATE $U \leftarrow $ top $k$ eigenvectors of $V_\pi$
		\IF{$\sum_{l=1}^k \lambda_l(V_\pi) - \langle \Omega \,|\, V_\pi \rangle \leq \epsilon \langle \Omega \,|\, V_\pi \rangle$}
		\STATE{break}
		\ENDIF
		\STATE $\widehat\Omega \leftarrow U \diag\left([\ones_k, \zeros_{d-k}]\right) U^T$
		\STATE $\tau = 2/(2+t)$
		\STATE $\Omega \leftarrow (1-\tau)\Omega + \tau \widehat\Omega$
  		\ENDFOR
		\STATE{\bfseries Output:} $\Omega, \pi, \langle \Omega \,|\, V_\pi \rangle$
	\end{algorithmic}
\end{algorithm}

\subsection{Initialization and Stopping Criterion}

We propose to initialize Algorithms~\ref{alg:projected_supergradient}~and~\ref{alg:FW} with $\Omega_0 = UU^T$ where $U \in \R^{d \times k}$ is the matrix of the top $k$ eigenvectors (\emph{i.e.} the eigenvectors associated with the top $k$ eigenvalues) of $V_{\pi^*}$ and $\pi^*$ is an optimal transport plan between $\mu$ and $\nu$. In other words, $\Omega_0$ is the projection matrix onto the $k$ first principal components of the transport-weighted displacement vectors. Note that $\Omega_0$ would be optimal is $\pi^*$ were optimal for the min-max problem, and that this initialization only costs the equivalent of one iteration.\newline

When entropic regularization is used, Sinkhorn algorithm is run at each iteration of Algorithms~\ref{alg:projected_supergradient}~and~\ref{alg:FW}. We propose to initialize the potentials in Sinkhorn algorithm with the latest computed potentials, so that the number of iterations in Sinkhorn algorithm should be small after a few iterations of Algorithms~\ref{alg:projected_supergradient}~or~\ref{alg:FW}. \newline

We sometimes need to compute $\ProjWasserstein{k}(\mu,\nu)$ for all $k \in \range{d}$, for example to choose the optimal $k$ with an ``elbow'' rule. To speed the computations up, we propose to compute this sequence iteratively from $k=d$ to $k=1$. At each iteration, \emph{i.e.} for each dimension $k$, we initialize the algorithm with $\Omega_0 = UU^T$, where $U \in \R^{d \times k}$ is the matrix of the top $k$ eigenvectors of $V_{\pi_{k+1}}$ and $\pi_{k+1}$ is the optimal transport plan for dimension $k+1$. We also initialize the Sinkhorn algorithm with the latest computed potentials.

Instead of running a fixed number of iterations in Algorithm~\ref{alg:FW}, we propose to stop the algorithm when the computation error is smaller than a fixed threshold $\epsilon$. The computation error at iteration $t$ is:
\[
	\frac{| \ProjWasserstein{k}(\mu, \nu) - \widehat{\ProjWasserstein{k}}(t)|}{\ProjWasserstein{k}(\mu, \nu)} \leq \frac{\Delta(t)}{\widehat{\ProjWasserstein{k}}(t)}
\]
where $\widehat{\ProjWasserstein{k}}(t)$ is the computed ``max-min'' value and $\Delta(t)$ is the duality gap at iteration $t$. We stop as soon as $\Delta(t) / \widehat{\ProjWasserstein{k}}(t) \leq \epsilon$.

\section{Experiments}\label{sec:exp}

We first compare SRW with the experimental setup used to evaluate FactoredOT~\citep{weed19}. We then study the ability of SRW distances to capture the dimension of sampled measures by looking at their value for increasing dimensions $k$, as well as their robustness to noise.

\subsection{Fragmented Hypercube}

We first consider $\mu = \mathcal{U}([-1,1])^d$ to be uniform over an hypercube, and $\nu = T_\#\mu$ the pushforward of $\mu$ under the map $T(x) = x + 2\,\textrm{sign}(x) \odot ( \sum_{k=1}^{k^*} e_k)$, where sign is taken elementwise, $k^* \in \range{d}$ and $(e_1, ..., e_d)$ is the canonical basis of $\Rd$. The map $T$ splits the hypercube into four different hyperrectangles. $T$ is a subgradient of a convex function, so by~\citeauthor{Brenier91}'s theorem \citeyearpar{Brenier91} it is an optimal transport map between $\mu$ and $\nu = T_\#\mu$ and
\[
	\Wd^2(\mu, \nu) = \int \|x - T(x)\|^2 \,d\mu(x) = 4k^*.
\]
Note that for any $x$, the displacement vector $T(x) - x$ lies in the $k^*$-dimensional subspace $\textrm{span}\{e_1, ..., e_{k^*}\} \in \G_{k^*}$, which is optimal. This means that for $k \geq k^*$, $\ProjWasserstein{k}^2(\mu, \nu)$ is constant equal to $4k^*$. We show the interest of plotting, based on two empirical distributions $\hat\mu$ from $\mu$ and $\hat\nu$ from $\nu$, the sequence $k \mapsto \ProjWasserstein{k}^2(\hat\mu, \hat\nu)$, for different values of $k^*$. That sequence is increasing concave by proposition~\ref{prop:dependence_dimension}, and increases more slowly after $k=k^*$, as can be seen on Figure~\ref{fig:hypercube_curve_k}. This is the case because the last $d-k^*$ dimensions only represent noise, but is recovered in our plot.
\begin{figure}[h!]
	\centering
	\includegraphics[width=0.48\textwidth]{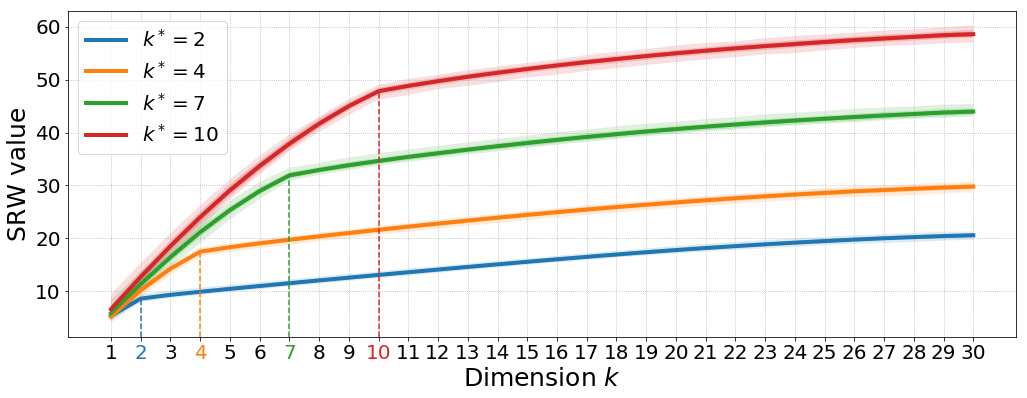}
	\vskip-.5cm
	\caption{$\ProjWasserstein{k}^2(\hat\mu, \hat\nu)$ depending on the dimension $k \in \range{d}$, for $k^* \in \{2, 4, 7, 10\}$, where $\hat\mu, \hat\nu$ are empirical measures from $\mu$ and $\nu$ respectively with $100$ points each. Each curve is the mean over $100$ samples, and shaded area show the min and max values.}
	\label{fig:hypercube_curve_k}
\end{figure}

We consider next $k^*=2$, and choose from the result of Figure~\ref{fig:hypercube_curve_k}, $k=2$. We look at the estimation error $|\Wd^2(\mu, \nu) - \ProjWasserstein{k}^2(\hat\mu, \hat\nu)|$ where $\hat\mu, \hat\nu$ are empirical measures from $\mu$ and $\nu$ respectively with $n$ points each. In Figure~\ref{fig:hypercube_estimation_error}, we plot this estimation error depending on the number of points $n$. In Figure~\ref{fig:hypercube_subspace_estimation_error}, we plot the subspace estimation error $\|\Omega^* - \widehat\Omega\|$ depending on $n$, where $\Omega^*$ is the optimal projection matrix onto $\textrm{span}\{e_1, e_2\}$.

\begin{figure}[!h]
	\centering
	\includegraphics[width=0.48\textwidth]{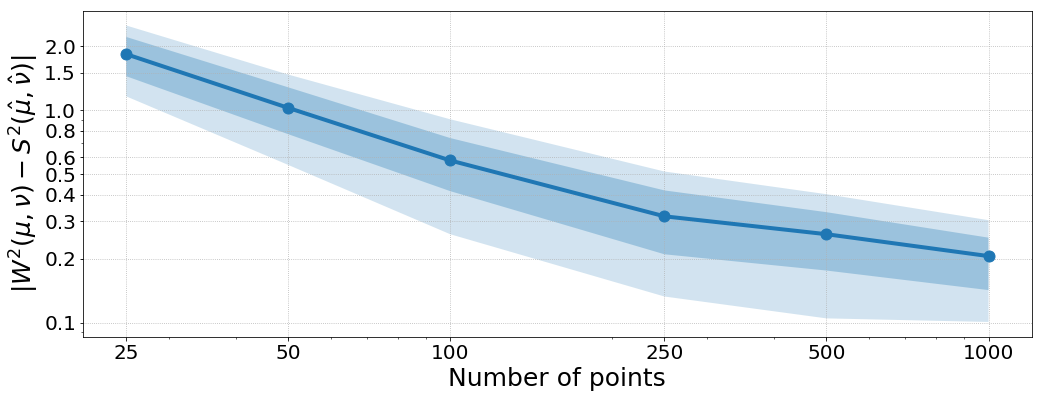}
	\vskip-.5cm
	\caption{Mean estimation error over $500$ random samples for $n$ points, $n \in \{25, 50, 100, 250, 500, 1000\}$. The shaded areas represent the 10\%-90\% and 25\%-75\% quantiles over the $500$ samples.}
	\label{fig:hypercube_estimation_error}
\end{figure}

\begin{figure}[!h]
	\centering
	\includegraphics[width=0.48\textwidth]{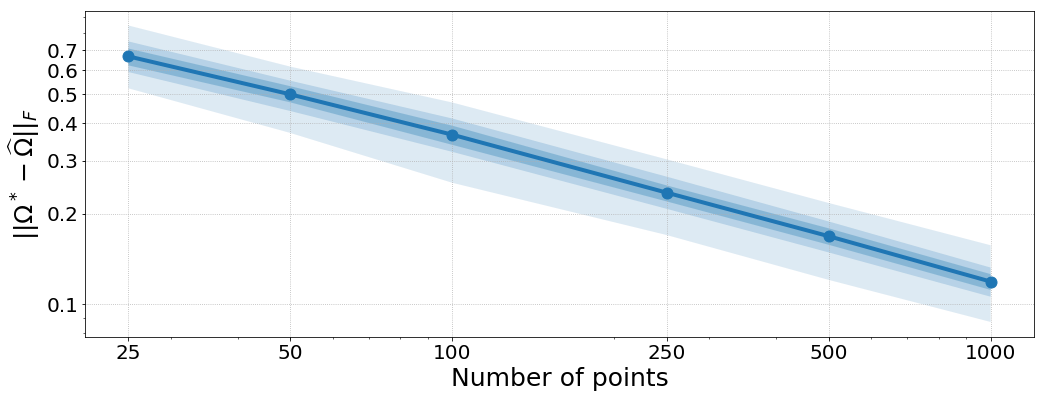}
	\vskip-.5cm
	\caption{Mean estimation of the subspace estimation error over $500$ samples, depending on $n \in \{25, 50, 100, 250, 500, 1000\}$. The shaded areas represent the 10\%-90\% and 25\%-75\% quantiles over the $500$ samples.}
	\label{fig:hypercube_subspace_estimation_error}
\end{figure}

We also plot the optimal transport plan (in the sense of $\Wd$, Figure~\ref{fig:hypercube_map} left) and the optimal transport plan (in the sense of $\ProjWasserstein{2}$) between $\hat\mu$ and $\hat\nu$ (with $n=250$ points each, Figure~\ref{fig:hypercube_map} right).
\begin{figure}[h!]
	\centering
	\begin{subfigure}
       		 \centering
		\includegraphics[width=0.22\textwidth]{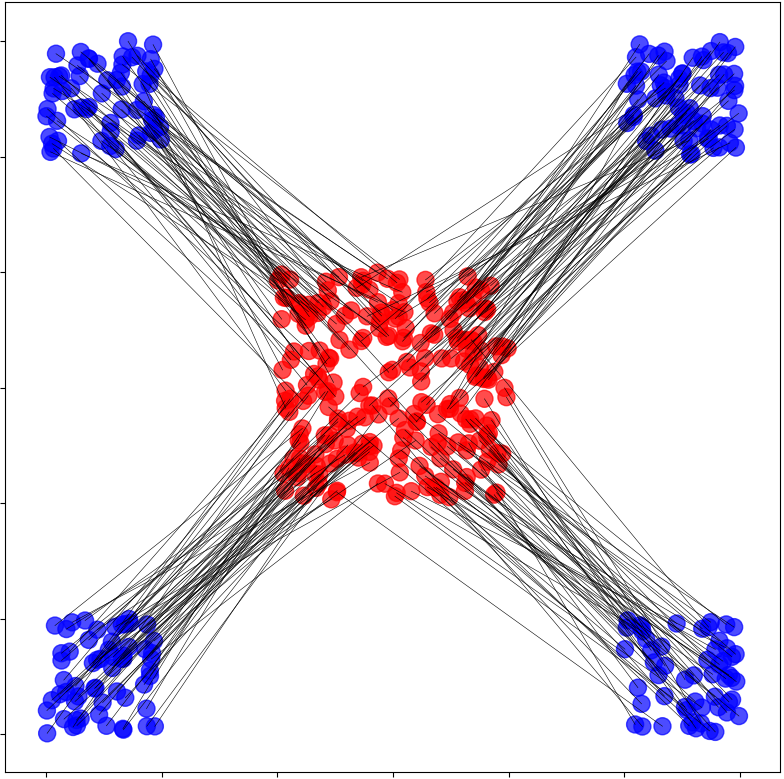}
    	\end{subfigure}%
    	~ 
   	 \begin{subfigure}
      		\centering
        		\includegraphics[width=0.22\textwidth]{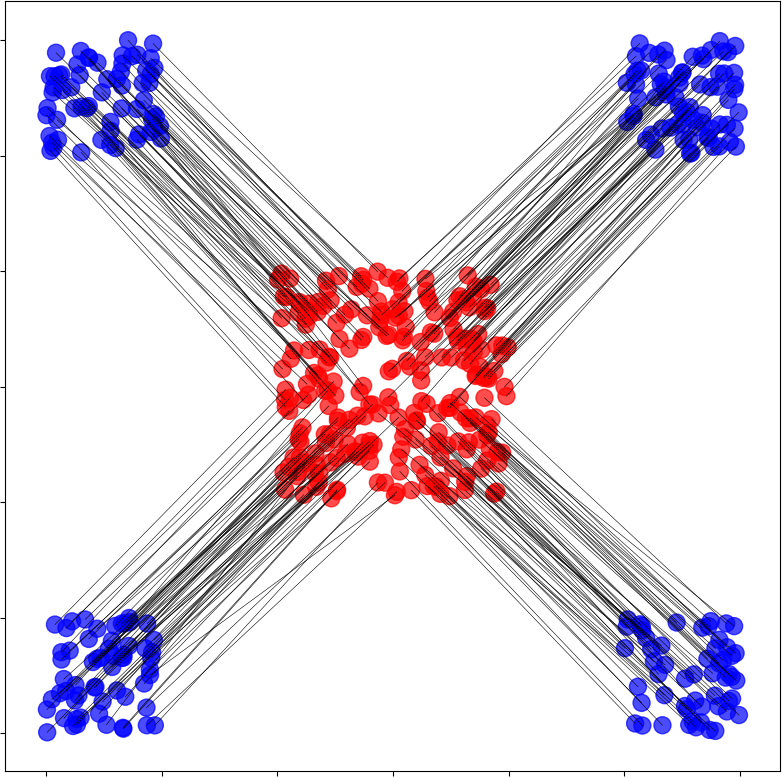}
   	 \end{subfigure}
\vskip-.5cm    \caption{Fragmented hypercube, $n=250$, $d=30$. Optimal mapping in the Wasserstein space (left) and in the SRW space (right). Geodesics in the SRW space are robust to statistical noise.}
		\label{fig:hypercube_map}
\end{figure}

\subsection{Robustness, with $20$-D Gaussians} \label{expe:gaussians}

We consider $\mu = \mathcal{N}(0, \Sigma_1)$ and $\nu = \mathcal{N}(0, \Sigma_2)$, with $\Sigma_1, \Sigma_2 \in \Rdd$ semidefinite positive of rank $k$. It means that the supports of $\mu$ and $\nu$ are $k$-dimensional subspaces of $\Rd$. Although those two subspaces are $k$-dimensional, they may be different. Since the union of two $k$-dimensional subspaces is included in a $2k$-dimensional subspace, for any $l \geq 2k$, $\ProjWasserstein{l}^2(\mu, \nu) = \Wd^2(\mu,\nu)$. \newline

For our experiment, we simulated $100$ independent couples of covariance matrices $\Sigma_1, \Sigma_2$ in dimension $d=20$, each having independently a Wishart distribution with $k=5$ degrees of freedom. For each couple of matrices, we draw $n=100$ points from $\mathcal{N}(0, \Sigma_1)$ and $\mathcal{N}(0, \Sigma_2)$ and considered $\hat\mu$ and $\hat\nu$ the empirical measures on those points. In Figure~\ref{fig:gaussians_curve_k}, we plot the mean (over the $100$ samples) of $l \mapsto \ProjWasserstein{l}^2(\hat\mu, \hat\nu) / \Wd^2(\hat\mu, \hat\nu)$. We plot the same curve for noisy data, where each point was added a $\mathcal{N}(0,I)$ random vector. With moderate noise, the data is only approximately on the two $k=5$-dimensional subspaces, but the SRW does not vary too much.
\begin{figure}[h!]
	\centering
	\includegraphics[width=0.48\textwidth]{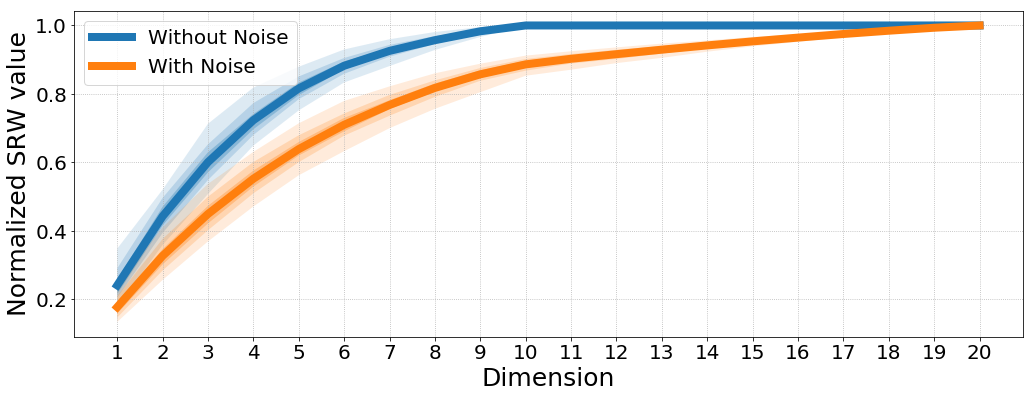}
	\vskip-.5cm
	\caption{Mean normalized SRW distance, depending on the dimension. The shaded area show the 10\%-90\% and 25\%-75\% quantiles and the minimum and maximum values over the $100$ samples.}
	\label{fig:gaussians_curve_k}
\end{figure}

\subsection{$\ProjWasserstein{k}$ is Robust to Noise}
As in experiment~\ref{expe:gaussians}, we consider $100$ independent samples of couples $\Sigma_1, \Sigma_2 \in \Rdd$, each following independently a Wishart distribution with $k=5$ degrees of freedom. For each couple, we draw $n=100$ points from $\mathcal{N}(0,\Sigma_1)$ and $\mathcal{N}(0,\Sigma_2)$ and consider the empirical measures $\hat\mu$ and $\hat\nu$ on those points. We then gradually add Gaussian noise $ \sigma \mathcal{N}(0, I)$ to the points, giving measures $\hat\mu_\sigma$, $\hat\nu_\sigma$. In Figure~\ref{fig:gaussians_curve_noise}, we plot the mean (over the $100$ samples) of the relative errors 
\[
	\sigma \mapsto \frac{| \ProjWasserstein{5}^2(\hat\mu_\sigma, \hat\nu_\sigma) - \ProjWasserstein{5}^2(\hat\mu_0, \hat\nu_0) |}{ \ProjWasserstein{5}^2(\hat\mu_0, \hat\nu_0) }
\]
and
\[
	\sigma \mapsto \frac{| \Wd^2(\hat\mu_\sigma, \hat\nu_\sigma) - \Wd^2(\hat\mu_0, \hat\nu_0) |}{ \Wd^2(\hat\mu_0, \hat\nu_0) }.
\]
Note that for small noise level, the imprecision in the computation of the SRW distance adds up to the error caused by the added noise. SRW distances seem more robust to noise than the Wasserstein distance when the noise has moderate to high variance.\newline

\begin{figure}[h!]
	\centering
	\includegraphics[width=0.48\textwidth]{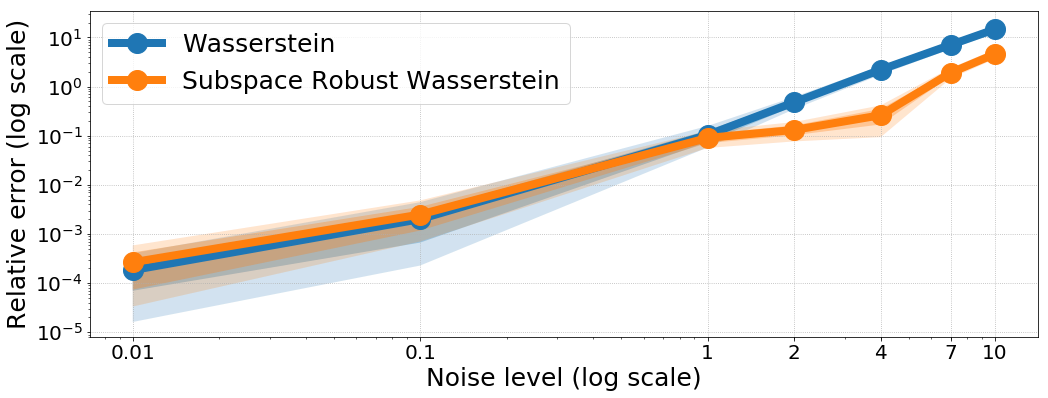}
	\vskip-.5cm
	\caption{Mean SRW distance over $100$ samples, depending on the noise level. Shaded areas show the min-max values and the 10\%-90\% quantiles.}
	\label{fig:gaussians_curve_noise}
\end{figure}

\subsection{Computation time}

We consider the Fragmented Hypercube experiment, with increasing dimension $d$ and fixed $k^*=2$. Using $k=2$ and Algorithm~\ref{alg:FW} with $\gamma = 0.1$ and stopping threshold $\epsilon=0.05$, we plot in Figure~\ref{fig:times} the mean computation time of both SRW and Wasserstein distances on GPU, over $100$ random samplings with $n=100$. It shows that SRW computation is quadratic in dimension $d$, because of the eigendecomposition of matrix $V_\pi$ in Algorithm~\ref{alg:FW}.
\begin{figure}[h!]
	\centering
	\includegraphics[width=0.48\textwidth]{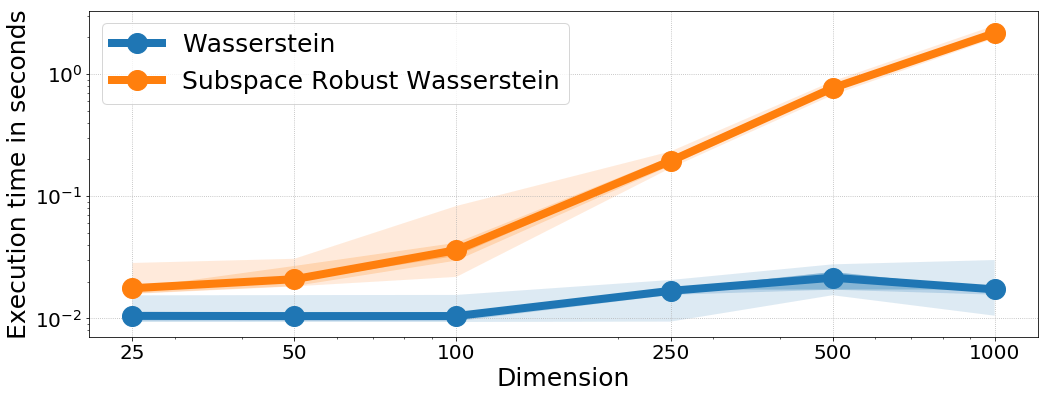}
	\vskip-.5cm
	\caption{Mean computation times on GPU (log-log scale). The shaded areas show the minimum and maximum values over the $100$ experiments.}
	\label{fig:times}
\end{figure}

\subsection{Real Data Experiment}\label{exp:real}

We consider the scripts of seven movies. Each script is transformed into a list of words, and using word2vec \cite{mikolov2018advances}, into a measure over $\R^{300}$ where the weights correspond to the frequency of the words. We then compute the SRW distance between all pairs of films: see Figure~\ref{fig:table_distances_cinema} for the SRW values. Movies with a same genre or thematic tend to be closer to each other: this can be visualized by running a two-dimensional metric multidimensional scaling (mMDS) on the SRW distances, as shown in Figure~\ref{fig:cinema_mmds} (left).
\begin{figure}[h!]
\setlength\tabcolsep{0.13cm}
\begin{tabular}{|c|c|c|c|c|c|c|c|}
  \hline
   		& \textit{D} & \textit{G} & \textit{I} & \textit{KB1} & \textit{KB2} & \textit{TM} & \textit{T} \\ \hline
  \textit{D} 		& 0 	     & 0.186 & 0.186 & 0.195 & 0.203 & 0.186 & \textbf{0.171} \\ \hline
  \textit{G} 		& 0.186 & 0 & \textbf{0.173} & 0.197 & 0.204 & 0.176 & 0.185 \\ \hline
  \textit{I} 		& 0.186 & 0.173 & 0 & 0.196 & 0.203 & \textbf{0.171} & 0.181 \\ \hline
  \textit{KB1} 	& 0.195 & 0.197 & 0.196 & 0 & \textbf{0.165} & 0.190 & 0.180 \\ \hline
  \textit{KB2} 	& 0.203 & 0.204 & 0.203 & \textbf{0.165} & 0 & 0.194 & 0.180 \\ \hline
  \textit{TM} 		& 0.186 & 0.176 & \textbf{0.171} & 0.190 & 0.194 & 0 & 0.183 \\ \hline
  \textit{T}		& \textbf{0.171} & 0.185 & 0.181 & 0.180 & 0.180 & 0.183 & 0 \\ \hline
\end{tabular}
\vskip-.3cm
\caption{\label{fig:table_distances_cinema}$\ProjWasserstein{k}^2$ distances between different movie scripts. Bold values correspond to the minimum of each line. \textit{D}=Dunkirk, \textit{G}=Gravity, \textit{I}=Interstellar, \textit{KB1}=Kill Bill Vol.1, \textit{KB2}=Kill Bill Vol.2, \textit{TM}=The Martian, \textit{T}=Titanic.}
\end{figure}

In Figure~\ref{fig:cinema_words} (right), we display the projection of the two measures associated with films \textit{Kill Bill Vol.1} and \textit{Interstellar} onto their optimal subspace. We compute the first (weighted) principal component of each projected measure, and find among the whole dictionary their 5 nearest neighbors in terms of cosine similarity. For \textit{Kill Bill Vol.1} , these are: 'swords', 'hull', 'sword', 'ice', 'blade'. For \textit{Interstellar}, they are: 'spacecraft', 'planets', 'satellites', 'asteroids', 'planet'. The optimal subspace recovers the semantic dissimilarities between the two films.

\begin{figure}[h!]
	\centering
	\begin{subfigure}
       		 \centering
		 \includegraphics[width=0.22\textwidth]{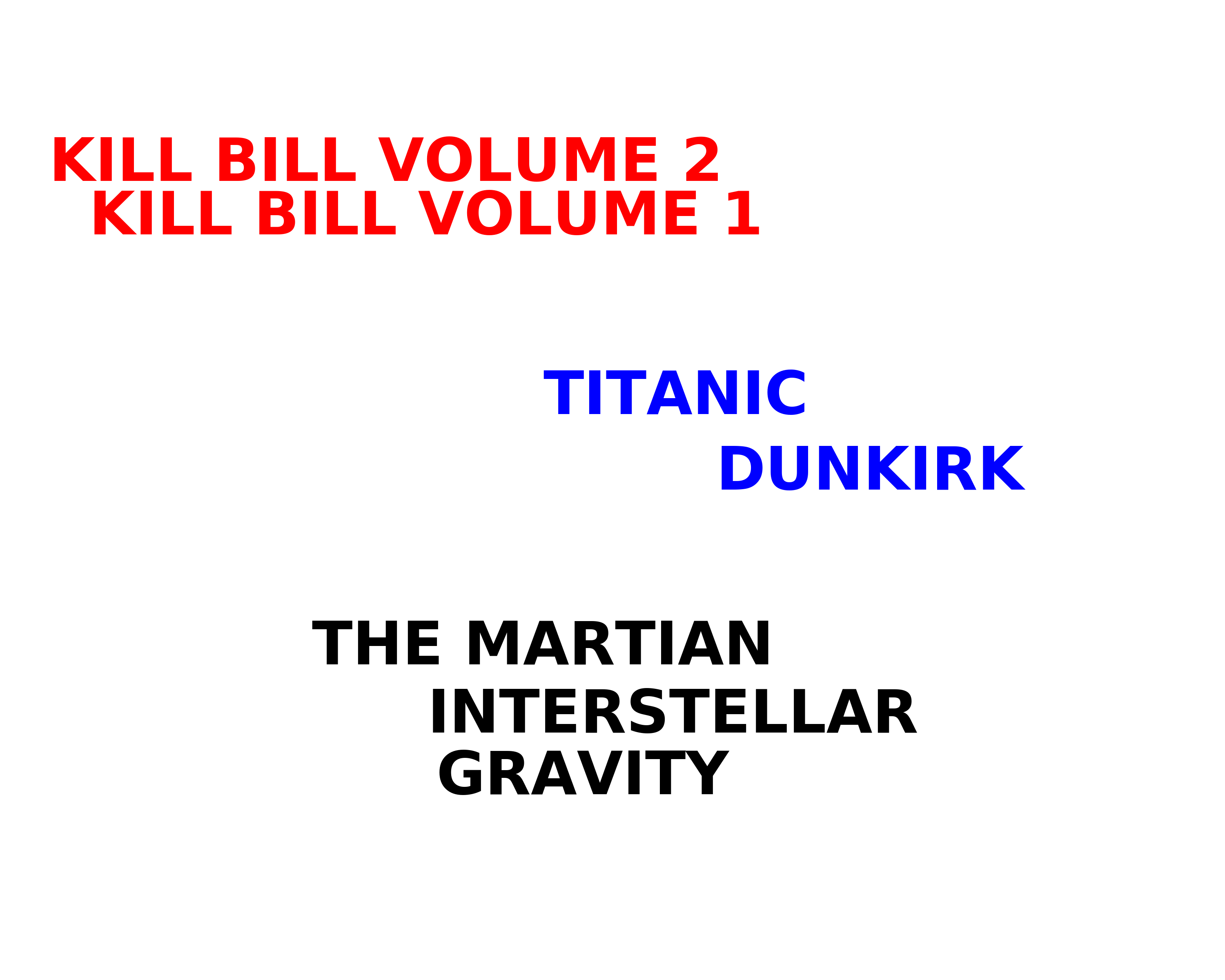}
    	\end{subfigure}%
    	~ 
   	 \begin{subfigure}
      		\centering
        		\includegraphics[width=0.22\textwidth]{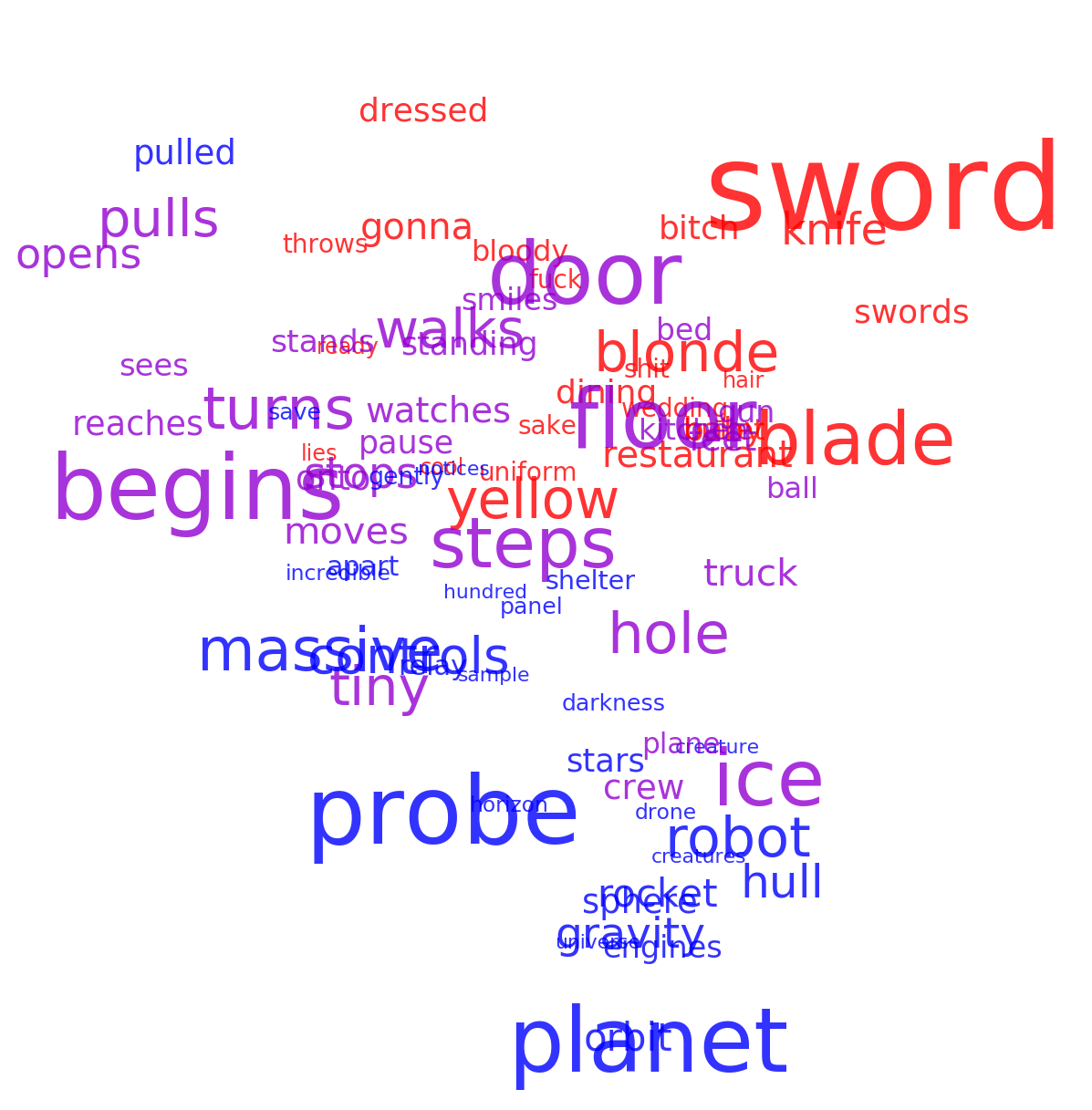}
   	 \end{subfigure}
	\vskip-.5cm
	\caption{\emph{Left:} Metric MDS projection for the distances of Figure~\ref{fig:table_distances_cinema}. \emph{Right:} Optimal 2-dimensional projection between \textit{Kill Bill Vol.1} (red) and \textit{Interstellar} (blue). Words appearing in both scripts are displayed in violet. For clarity, only the 30 most frequent words of each script are displayed.} \label{fig:cinema_mmds} \label{fig:cinema_words}
\end{figure}
\section{Conclusion}
We have proposed in this paper a new family of OT distances with robust properties. These distances take a particular interest when used with a squared-Euclidean cost, in which case they have several properties, both theoretical and computational. These distances share important properties with the $2$-Wasserstein distance, yet seem far more robust to random perturbation of the data and able to capture better signal. We have provided algorithmic tools to compute these SRW distance. They come at a relatively modest overhead, given that they require using regularized OT as the inner loop of a FW type algorithm. Future work includes the investigation of even faster techniques to carry out these computations, eventually automatic differentiation schemes as those currently benefitting the simple use of Sinkhorn divergences.

\paragraph{Acknowledegments.} Both authors acknowledge the support of a "Chaire d'excellence de l'IDEX Paris Saclay". We thank P. Rigollet and J. Weed for their remarks and their valuable input.

\newpage

\clearpage
\appendix

\begin{figure*}[!ht]
	\includegraphics[width=\textwidth]{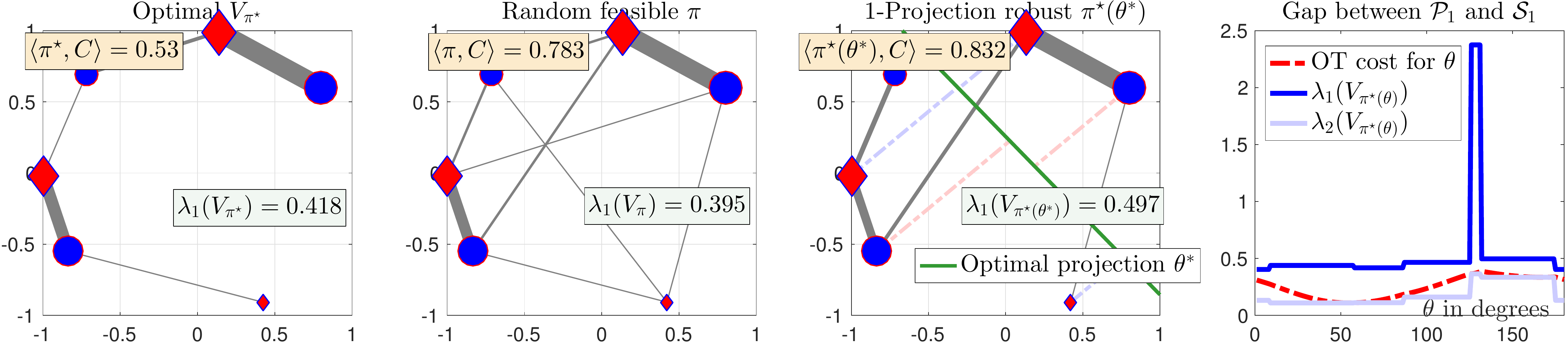}
	\caption{This figure should be compared to Figure~\ref{fig:lasuperfigure}. We also present an example for which the explicit computation of projection $\ProjPWasserstein{k}$ and subspace $\ProjWasserstein{k}$ robust Wasserstein distances described in \S3 can be carried out explicitly, by simple enumeration. Unlike Figure~\ref{fig:lafigurenulle}, and as can be seen in the rightmost plot, these two quantities do not coincide here. That plot reveals that the minimum across all  maximal eigenvalues of second order moment matrices computed on all optimal OT plans obtained by enumerating all lines (the subspace robust quantity) is strictly larger than the worst possible projection cost.}\label{fig:lafigurenulle}
\end{figure*}

\section{Projection Robust Wasserstein Distances}
In this section, we prove some basic properties of projection robust Wasserstein distances $\ProjPWasserstein{k}$. First note that the definition of $\ProjPWasserstein{k}$ makes sense, since for any $\mu,\nu \in \PdRd$, $k \in \range{d}$ and $E \in \G_k$, ${P_E}_\#\mu$ and ${P_E}_\#\nu$ have a second moment (for orthogonal projections are $1$-Lipschitz).

$\ProjPWasserstein{k}$ is also well posed, since one can prove the existence of a maximizing subspace. To prove this, we will need the following lemma stating that the admissible set of couplings between the projected measures are exactly the projections of the admissible couplings between the original measures:
\begin{lemma} \label{lemma:projected_transport_plans}
Let $f: \Rd \to \Rd$ Borel and $\mu,\nu \in \PRd$. Then $\Pi( f_\#\mu, f_\#\nu ) = \left\{ (f \otimes f)_\#\pi \,|\, \pi \in \Pi(\mu, \nu) \right\}$.
\end{lemma}
This can be used to get the following result:
\begin{proposition} \label{prop:existence_optimal_subspaces}
For $\mu, \nu \in \PdRd$ and $k \in \range{d}$, there exists a subspace $E^* \in \G_k$ such that
\[
	\ProjPWasserstein{k}(\mu, \nu) = \Wd\left({P_{E^*}}_\#\mu , {P_{E^*}}_\#\nu \right).
\]
\end{proposition}
\begin{proof}
The Grassmannian $\G_k$ is compact, and we show that the application $E \mapsto \Wd\left( {P_E}_\#\mu , {P_E}_\#\nu \right)$ is upper semicontinuous, which gives existence.
\end{proof}

Note that we could define projection robust Wasserstein distances for any $p \geq 1$ by:
\[
	\sup_{E \in \G_k} \Wd_p \left( {P_E}_\#\mu , {P_E}_\#\nu \right).
\]
Then there is still existence of optimal subspaces, and it defines a distance over
\[
	\PpRd = \left\{ \mu \in \PRd \,\bigg|\, \int \|x\|^p \,d\mu(x) < \infty \right\}.
\]

To prove the identity of indiscernibles, we use the following Lemma due to R\'enyi, generalizing Cram\'er-Wold theorem:
\begin{lemma}\label{lemma:renyi-cramer-wold}
Let $\left(E_j\right)_{j \in J}$ be a family of subspaces of $\Rd$ such that $\bigcup_{j \in J} E_j= \Rd$. Let $\mu, \nu \in \PRd$ such that for all $j \in J$, ${P_{E_j}}_\#\mu = {P_{E_j}}_\#\nu$. Then $\mu = \nu$.
\end{lemma}

\newpage
\section{Proofs}

\paragraph{Proof of Lemma~\ref{lemma:infsup_is_minmax}.}
For $\pi \in \Pi(\mu, \nu)$, the application $E \mapsto \int \|P_E(x)-P_E(y)\|^2 \,d\pi(x,y)$ is continuous and $\G_k$ is compact, so the supremum is a maximum. Moreover, the application $\pi \mapsto \max_{E \in \G_k} \int \|P_E(x)-P_E(y)\|^2 \,d\pi(x,y)$ is lower semicontinuous as the maximum of lower semicontinuous functions. Since $\Pi(\mu, \nu)$ is compact (for any sequence in $\Pi(\mu,\nu)$ is tight), the infimum is a minimum.

\paragraph{Proof of Lemma~\ref{lemma:dual_formulation}.}
A classical variational result by~\cite{fan1949theorem} states that
\[
	 \sum_{l=1}^k \lambda_l(V_\pi) = \max_{\substack{U \in \R^{k \times d}\\UU^T = I_k}} \tr\left( U V_\pi U^T \right).
\]
Then using the linearity of the trace:
\[
	\begin{split}
	 	\sum_{l=1}^k \lambda_l(V_\pi) &= \max_{\substack{U \in \R^{k \times d}\\UU^T = I_k}} \int \tr\left[U(x-y)(x-y)^TU^T\right] \,d\pi(x,y) \\
		&= \max_{\substack{U \in \R^{k \times d}\\UU^T = I_k}} \int \|U(x-y)\|^2 \,d\pi(x,y) \\
		&= \max_{E \in \G_k} \int \|P_E(x) - P_E(y)\|^2 \,d\pi(x,y).
	\end{split}
\]
Taking the minimum over $\pi \in \Pi(\mu,\nu)$ yields the result.

\paragraph{Proof of Theorem~\ref{theo:convex_relaxation}.}

$\ProjWasserstein{k}^2(\mu, \nu)$ = \eqref{eqn:minmaxOmega} :
We fix $\pi \in \Pi(\mu, \nu)$ and focus on the inner maximization in~\eqref{eqn:minmaxOmega} :
\[
	\max_{\substack{0 \preceq \Omega \preceq I\\\tr(\Omega) = k}}  \int d_\Omega^2 \,d\pi = \max_{\substack{0 \preceq \Omega \preceq I\\\tr(\Omega) = k}}  \left\langle \Omega \,|\, V_\pi \right\rangle.
\]
A result by~\cite{overton1993optimality} shows that this is equal to
\[
	\max_{\substack{U \in \R^{k \times d}\\UU^T = I_k}}  \tr\left( U V_\pi U^T \right)
\]
which is nothing but the sum of the $k$ largest eigenvalues of $V_\pi$ by Fan's result. By lemma~\ref{lemma:dual_formulation}, taking the minimum over $\pi \in \Pi(\mu, \nu)$ yields the result.

\eqref{eqn:minmaxOmega} = \eqref{eqn:maxOmegamin} : We will use Sion's minimax theorem to interchange the minimum and the maximum. Put $f(\Omega, \pi) = \int d_\Omega^2 \,d\pi$ and $\mathcal{R} = \left\{ \Omega \in \Rdd \,|\, 0 \preceq \Omega \preceq I \,;\, \tr(\Omega) = k \right\}$. Note that $\mathcal{R}$ is convex and compact, and $\Pi(\mu,\nu)$ is convex (and actually compact, but we won't need it here). Moreover, $f$ is bilinear and for any $\pi \in \Pi(\mu,\nu)$, $f(\cdot, \pi)$ is continuous. Let $\Omega \in \mathcal{R}$. Let us show that $f(\Omega, \cdot)$ is lower semicontinuous for the weak convergence. Let $(\phi_j)_{j \in \N}$ be an increasing sequence of bounded continuous functions, converging pointwise to $d_\Omega^2$. Then $f(\Omega, \pi) = \sup_{j \in \N} \int \phi_j \,d\pi$. For $j \in \N$, $\phi_j$ is continuous and bounded, so $\pi \mapsto \int \phi_j \,d\pi$ is continuous for the weak convergence. Then $f(\Omega, \cdot)$ is lower semicontinuous as the supremum of continuous functions.
Then by Sion's minimax theorem,
\[
	 \min_{\pi \in \Pi(\mu,\nu)} \max_{\Omega \in \mathcal{R}} f(\Omega, \pi) = \max_{\Omega \in \mathcal{R}} \min_{\pi \in \Pi(\mu,\nu)} f(\Omega, \pi)
\]
which is exactly \eqref{eqn:minmaxOmega} = \eqref{eqn:maxOmegamin}.

\eqref{eqn:maxOmegamin} = \eqref{eqn:maxOmegaWass} : Fix $\Omega \in \mathcal{R}$. One has:
\[
	\begin{split}
		&\min_{\pi \in \Pi(\mu, \nu)} \int d_\Omega^2 \,d\pi = \min_{\pi \in \Pi(\mu, \nu)} \int \|\Omega^{1/2}(x-y)\|^2 \,d\pi(x,y) \\
		&= \min_{\pi \in \Pi(\mu, \nu)} \int \|x-y\|^2 \,d\left[\Omega^{1/2} \otimes \Omega^{1/2} \right]_\#\pi(x,y) \\
		&= \min_{\rho \in \Pi(\Omega^{1/2}_\#\mu,\, \Omega^{1/2}_\#\nu)} \int \|x-y\|^2 \,d\rho(x,y) \\
		&= \Wd^2\left(\Omega^{1/2}_\#\mu, \Omega^{1/2}_\#\nu\right)
	\end{split}
\]
where we have used Lemma~\ref{lemma:projected_transport_plans}. Taking the maximum over $\Omega \in \mathcal{R}$ gives the result.


\paragraph{Proof of Lemma~\ref{lemma:tight_bounds}.}

We use the fact that the pushforward by $f$ of a Dirac at $x$ is the Dirac at $f(x)$, and that the $\Wd$ distance between two Diracs is the Euclidean distance between the points:
\[
	\begin{split}
		\ProjWasserstein{k}(\delta_x, \delta_y) &= \max_{\substack{0 \preceq \Omega \preceq I\\\tr(\Omega) = k}} \Wd \left( {\Omega^{1/2}}_\#\delta_x , {\Omega^{1/2}}_\#\delta_y \right) \\
		&= \max_{\substack{0 \preceq \Omega \preceq I\\\tr(\Omega) = k}} \| {\Omega^{1/2}} (x-y) \|.
	\end{split}
\]
Since $\| {\Omega^{1/2}} (x-y) \| \leq \|x-y\|$ with equality for any orthogonal projection matrix $\Omega$ onto a subspace $E \in \G_k$ such that $\textrm{span}(y-x) \subset E$, the result follows.


\paragraph{Proof of Proposition~\ref{prop:equivalent_metrics}.}
Let $k \in \range{d}$ and $\mu, \nu \in \PdRd$. Let us prove the upper bound on $\ProjWasserstein{k}$. Using the change of variable formula and the fact that for any $\Omega \in \left\{ \Omega \in \Rdd \,|\, 0 \preceq \Omega \preceq I \,;\, \tr(\Omega) = k \right\}$, $\Omega^{1/2}$ is $1$-Lipschitz,
\[
	\begin{split}
 		&\ProjWasserstein{k}^2(\mu, \nu) \\
 		&= \displaystyle \adjustlimits \max_{\substack{0 \preceq \Omega \preceq I\\\tr(\Omega) = k}} \min_{\pi \in \Pi(\mu, \nu)} \int \|\Omega^{1/2}(x-y)\|^2 \,d\pi(x,y) \\
		&\leq \displaystyle \adjustlimits \max_{\substack{0 \preceq \Omega \preceq I\\\tr(\Omega) = k}} \min_{\pi \in \Pi(\mu, \nu)} \int \|x - y\|^2 \,d\pi(x,y) \\
		&= \Wd^2(\mu, \nu)
	\end{split}
\]
which gives the upper bound. For the lower bound, we define $\mathcal{B}_k \subset \G_k$ the (finite) set of $k$-dimensional subspaces of $\Rd$ spanned by $k$ vectors of the canonical basis of $\Rd$:
\[
	\mathcal{B}_k = \left\{ \textrm{span}(e_{\sigma(1)}, ..., e_{\sigma(k)}) \,|\, \sigma \in \mathfrak{S}_d \right\}.
\]
Let us now bound $\ProjWasserstein{k}$ from below:
\[
	\begin{split}
 		&\ProjWasserstein{k}^2(\mu, \nu) \\
    		&= \adjustlimits \min_{\pi \in \Pi(\mu, \nu)} \max_{\substack{0 \preceq \Omega \preceq I\\\tr(\Omega) = k}} \int \|\Omega^{1/2}(x-y)\|^2 \,d\pi(x,y) \\
    		&\geq \adjustlimits \min_{\pi \in \Pi(\mu, \nu)}\max_{E \in \mathcal{B}_k} \int \|P_E(x)-P_E(y)\|^2 \,d\pi(x,y) \\
    		&= \min_{\pi \in \Pi(\mu, \nu)} \max_{\substack{A \subset \range{d}\\|A|=k}} \int \sum_{i \in A} (x_i - y_i)^2 \,d\pi(x,y) \\
    		&= \min_{\pi \in \Pi(\mu, \nu)} \max_{\substack{A \subset \range{d}\\|A|=k}} \sum_{i \in A} \int (x_i - y_i)^2 \,d\pi(x,y).
	\end{split}
\]
For $\pi \in \Pi(\mu,\nu)$,
\[
	\max_{\substack{A \subset \range{d}\\|A|=k}} \sum_{i \in A} \int (x_i - y_i)^2 \,d\pi(x,y)
\]
is the sum of the $k$ largest elements of $I=\left\{ \int (x_i - y_i)^2 \,d\pi(x,y) \,|\, i \in \range{d} \right\}$, so it is greater than $\frac{k}{d}$ times the sum of all the elements in $I$:
\[
	\begin{split}
 		\ProjWasserstein{k}^2(\mu, \nu) &\geq \frac{k}{d} \min_{\pi \in \Pi(\mu, \nu)} \int \|x-y\|^2 \,d\pi(x,y) \\
  		&= \displaystyle \frac{k}{d} \Wd^2(\mu, \nu).
	\end{split}
\]

Note that in the case of $\mu=\delta_0$ and $\nu=\sigma$, the two inequalities in the proof of the lower bound are equalities, hence the tight lower bound constant.


\paragraph{Proof of Lemma~\ref{lemma:increasing_concave_in_K}.}
Let us first prove the lower bound:
\[
\begin{array}{ccl}
    \ProjWasserstein{k+1}^2(\mu,\nu)
    &=&
    \displaystyle \sum_{l=1}^k \lambda_l(V_{\pi_{k+1}}) + \lambda_{k+1}(V_{\pi_{k+1}})\\
    &\geq&
    \displaystyle \sum_{l=1}^k \lambda_l(V_{\pi_k}) + \lambda_{k+1}(V_{\pi_{k+1}}) \\
    &=&
    \ProjWasserstein{k}^2(\mu, \nu) + \lambda_{k+1}(V_{\pi_{k+1}}).
\end{array}
\]
Let us now prove the upper bound:
\[
\begin{array}{ccl}
    \ProjWasserstein{k+1}^2(\mu,\nu) 
    &=&
    \displaystyle \min_{\pi \in \Pi(\mu, \nu)} \sum_{l=1}^{k+1} \lambda_l(V_\pi) \\
    &\leq&
    \displaystyle \sum_{l=1}^{k+1} \lambda_l(V_{\pi_k}) \\
    &=&
    \ProjWasserstein{k}^2(\mu, \nu) + \lambda_{k+1}(V_{\pi_k}).
\end{array}
\]


\paragraph{Proof of Propositon~\ref{prop:dependence_dimension}.}
Increase is direct using lemma~\ref{lemma:increasing_concave_in_K}, since for any $\pi \in \Pi(\mu,\nu)$, $V_\pi$ has only nonnegative eigenvalues. \newline
Let $k \in \range{d-2}$. Then using twice lemma~\ref{lemma:increasing_concave_in_K},
\[
\begin{split}
    \ProjWasserstein{k+2}^2(\mu,\nu) &- \ProjWasserstein{k+1}^2(\mu,\nu) \\
    &\leq \lambda_{k+2}(V_{\pi_{k+1}}) \\
    &\leq \lambda_{k+1}(V_{\pi_{k+1}}) \\
    &\leq \ProjWasserstein{k+1}^2(\mu, \nu) - \ProjWasserstein{k}^2(\mu, \nu),
\end{split}
\]
which shows that $k \mapsto \ProjWasserstein{k}^2(\mu,\nu)$ is concave. \newline
Let $k \in \range{d-1}$. Although the minoration of $\ProjWasserstein{k+1}^2(\mu,\nu) - \ProjWasserstein{k}^2(\mu,\nu)$ is a direct consequence of concavity, we give a direct computation using lemma~\ref{lemma:increasing_concave_in_K}:
\[
\begin{split}
    &\ProjWasserstein{k+1}^2(\mu,\nu) - \ProjWasserstein{k}^2(\mu,\nu) \\
    &\geq \lambda_{k+1}(V_{\pi_{k+1}}) \\
    &\geq \displaystyle \frac{1}{d-k-1} \sum_{l=k+2}^d \lambda_l(V_{\pi_{k+1}}) \\
    &= \displaystyle \frac{1}{d-k-1} \left[ \sum_{l=1}^d \lambda_l(V_{\pi_{k+1}}) - \sum_{l=1}^{k+1} \lambda_l(V_{\pi_{k+1}})\right] \\
    &\geq \displaystyle \frac{1}{d-k-1} \left[ \Wd^2(\mu,\nu) - \ProjWasserstein{k+1}^2(\mu,\nu) \right],
\end{split}
\]
which implies that
\[
\begin{split}
    (d-k) &\left[ \ProjWasserstein{k+1}^2(\mu,\nu) - \ProjWasserstein{k}^2(\mu,\nu) \right] \\
    & \geq \Wd^2(\mu,\nu) - \ProjWasserstein{k}^2(\mu,\nu).
\end{split}
\]

Finally, the majoration of $\ProjWasserstein{k}(\mu,\nu)$ is a direct consequence of lemma~\ref{lemma:increasing_concave_in_K}:
\[
\begin{split}
    \ProjWasserstein{k+1}^2(\mu,\nu) &\leq \ProjWasserstein{k}^2(\mu,\nu) + \lambda_{k+1}(V_{\pi_k}) \\
    &\leq \ProjWasserstein{k}^2(\mu,\nu) + \lambda_k(V_{\pi_k}) \\
    &\leq \ProjWasserstein{k}^2(\mu,\nu) + \frac{1}{k} \sum_{l=1}^k \lambda_l(V_{\pi_k}) \\
    &= \frac{k+1}{k} \ProjWasserstein{k}^2(\mu,\nu).
\end{split}
\]


\paragraph{Proof of Proposition~\ref{prop:geodesic_space}.}
For $s,t \in [0,1]$, put $\pi(s,t) = (f_s, f_t)_\#\pi^* \in \Pi(\mu_s, \mu_t)$, which is our candidate for an optimal transport plan. Then
\[
  \begin{split}
    &\ProjWasserstein{k}^2(\mu_s, \mu_t) \leq \sum_{l=1}^k \lambda_l(V_{\pi(s,t)}) \\
    &= \sum_{l=1}^k \lambda_l\left\{ \int \left[ f_s(x,y) - f_t(x,y) \right] \right. \\
    &\qquad \qquad \qquad \left. \left[ f_s(x,y) - f_t(x,y) \right]^T d\pi^*(x,y) \vphantom{\int} \right\} \\
    &= \sum_{l=1}^k \lambda_l\left( (t-s)^2 V_{\pi^*} \right) \\
    &= (t-s)^2 \ProjWasserstein{k}^2(\mu, \nu)
  \end{split}
\]
where we have used 
\[
	\begin{split}
		f_s(x,y) - f_t(x,y) &= (1-s)x + sy - (1-t)x - ty\\
		&= (t-s)(x-y).
	\end{split}
\]
Then for $0 \leq s < t \leq 1$, using the triangular inequality,
\[
  \begin{split}
      &\ProjWasserstein{k}(\mu, \nu) \\
      &\leq \ProjWasserstein{k}(\mu, \mu_s) + \ProjWasserstein{k}(\mu_s, \mu_t) + \ProjWasserstein{k}(\mu_t, \nu) \\
      &\leq \left( s + (t-s) + (1-t) \right) \ProjWasserstein{k}(\mu,\nu) = \ProjWasserstein{k}(\mu,\nu)
  \end{split}
\]
which implies equality everywhere, and in particular optimality for $\pi(s,t)$. Then for all $s,t \in [0,1]$,
\[
    \ProjWasserstein{k}(\mu_s, \mu_t) = |t-s| \ProjWasserstein{k}(\mu, \nu),
\]
which shows that the curve $(\mu_t)$ has constant speed 
\[
	|\mu_t'| = \lim_{\epsilon \to 0} \frac{\ProjWasserstein{k}(\mu_{t+\epsilon}, \mu_t)}{|\epsilon|} = \ProjWasserstein{k}(\mu, \nu),
\]
and that the length of the curve $(\mu_t)$ is 
\[
	\begin{split}
		&\sup\left\{ \sum_{i=0}^{n-1} \ProjWasserstein{k}(\mu_{t_i}, \mu_{t_{i+1}}) \,\bigg|\, 
		\begin{array}{c}
			n \geq 1\\
			0 = t_0 < ... < t_n = 1
		\end{array}
		\right\}\\
		&= \ProjWasserstein{k}(\mu, \nu),
	\end{split}
\]
\emph{i.e.} that $(\mu_t)$ is a geodesic connecting $\mu$ and $\nu$.

\paragraph{Proof of Lemma~\ref{lemma:optimalOmega_fixedpi}.}
Although this is a direct consequence of~\cite{overton1993optimality}, we give an explicit proof.
Fix $\pi \in \Pi(\mu,\nu)$. Using the linearity of the trace,
\[
	\max_{\substack{0 \preceq \Omega \preceq I\\\tr(\Omega) = k}}  \int d_\Omega^2 \,d\pi = \max_{\substack{0 \preceq \Omega \preceq I\\\tr(\Omega) = k}} \tr(\Omega V_\pi),
\]
which is a SDP. Its dual writes
\[
	\min_{\substack{s \in \R, Z \in \Rdd\\ Z \succeq 0\\Z+sI \succeq V_\pi}} \tr(Z) + ks.
\]
Let us write the eigendecomposition of $V_\pi = U \textrm{diag}(\lambda_1, ..., \lambda_d) U^T$ with $\lambda_1 \geq ... \geq \lambda_d$. Put $\widehat{\Omega} = U \diag\left([\ones_k, \zeros_{d-k}]\right) U^T$, $\widehat{Z} = U \textrm{diag}((\lambda_1-\lambda_k) _+, ..., (\lambda_d - \lambda_k)_+)U^T$ and $\widehat{s} = \lambda_k$. Then $0 \preceq \widehat\Omega \preceq I$, $\tr(\widehat\Omega) = k$ and $(\widehat{s}, \widehat{Z})$ is admissible for the dual problem, with corresponding primal and dual values
\[
	\begin{split}
		&\tr(\widehat{\Omega} V_\pi) = \sum_{l=1}^k \lambda_l, \\
		&\tr(\widehat{Z}) + k\widehat{s} = \sum_{l=1}^k (\lambda_l - \lambda_k) + k\lambda_k = \sum_{l=1}^k \lambda_l.
	\end{split}	
\]
We found primal and dual admissible variables that give the same value, so these variables are optimal. In particular, $\widehat\Omega$ is solution to 
\[
	\max_{\substack{0 \preceq \Omega \preceq I\\\tr(\Omega) = k}}  \int d_\Omega^2 \,d\pi.
\]

\paragraph{Proof of Lemma~\ref{lemma:projected_transport_plans}.}
Let $E = \text{Im}(f) \subset \Rd$ and $\pi \in \Pi(\mu, \nu)$. Then $(f \otimes f)_\#\pi$ is an admissible transport plan between $f_\#\mu$ and $f_\#\nu$. Indeed, for any Borel set $A \subset E$, $(f \otimes f)_\#\pi(A \times E) = \pi(f^{-1}(A) \times f^{-1}(E)) = \pi(f^{-1}(A) \times \Rd) = \mu(f^{-1}(A)) = f_\#\mu(A)$, so $(f \otimes f)_\#\pi$ has first marginal $f_\#\mu$, and likewise, has second marginal $f_\#\nu$, \emph{i.e.} $(f \otimes f)_\#\pi \in \Pi( f_\#\mu, f_\#\nu )$. \newline
Conversely, let $\rho \in \Pi( f_\#\mu, f_\#\nu )$. Let us construct $\pi \in \Pi(\mu, \nu)$ such that $(f \otimes f)_\#\pi = \rho$. For any Borel sets $A,B \subset \Rd$, put
\[
	\pi(A \times B) = \frac{ \rho(f(A) \times f(B)) \mu(A) \nu(B)}{f_\#\mu(f(A)) \, f_\#\nu(f(B))}
\]
if $f_\#\mu(f(A)) \neq 0$ and $f_\#\mu(f(B)) \neq 0 $, and $\pi(A,B) = 0$ otherwise.
Then $\pi \in \Pi(\mu, \nu)$. Indeed, for any Borel set $A \subset \Rd$, $\pi(A \times \Rd) = \rho(f(A) \times E) \frac{\mu(A)}{f_\#\mu(f(A))} = \mu(A)$ if $f_\#\mu(f(A)) \neq 0$ and $\pi(A \times \Rd) = 0$ if $f_\#\mu(f(A)) = 0$. But then, $\mu(A) \leq \mu(f^{-1}(f(A))) =  f_\#\mu(f(A)) = 0$ so $\mu(A) = 0 = \pi(A \times \Rd)$. The same calculations give the result for the second marginal. \newline
There remains to prove that $(f \otimes f)_\#\pi = \rho$. For any Borel sets $A,B \subset E$, noting that $f(f^{-1}(A)) = A$ and $f(f^{-1}(B)) = B$, 
\[
	\begin{split}
		(f \otimes f)_\#\pi(A \times B) &= \pi(f^{-1}(A) \times f^{-1}(B)) \\
		&= \rho(A \times B) \frac{\mu(f^{-1}(A))}{f_\#\mu(A)} \frac{\nu(f^{-1}(B))}{f_\#\mu(B)} \\
		&= \rho(A \times B)
	\end{split}
\]
if $f_\#\mu(A) \neq 0$ and $f_\#\nu(B) \neq 0$. Otherwise, $(f \otimes f)_\#\pi(A \times B) = 0$ and $\rho(A \times B) \leq \min\{ \rho(A \times E), \rho(E \times B) \} = \min\{ f_\#\mu(A), f_\#\nu(B)\} = 0$, so $\rho(A \times B) = (f \otimes f)_\#\pi(A \times B) = 0$. \quad $\square$

\paragraph{Proof of Proposition~\ref{prop:existence_optimal_subspaces}.}
We endow the Grassmannian $\G_k$ with the metric topology associated with metric $d: (E,F) \mapsto \|P_E - P_F\|$, where $P_E$ and $P_F$ are the linear projectors onto $E$ and $F$. Then it is well known that $\G_k$ is compact under this topology. 

We only have to show that, for $\mu, \nu \in \PdRd$, the map $f: E \mapsto \Wd\left( {P_E}_\#\mu , {P_E}_\#\nu \right)$ is upper semicontinuous. For any orthogonal projector $P$, using lemma~\ref{lemma:projected_transport_plans},
\[
	\!\begin{split}
		\Wd^2\left( P_\#\mu , P_\#\nu \right) &= \min_{\rho \in \Pi(P_\#\mu , P_\#\nu)} \int \|x-y\|^2 \,d\rho(x,y) \\
		&= \min_{\pi \in \Pi(\mu,\nu)} \int \|x-y\|^2 \,d(P_E\otimes P_E)_\#\pi(x,y) \\
		&= \min_{\pi \in \Pi(\mu,\nu)} \int \|P(x-y)\|^2 \,d\pi(x,y).
	\end{split}
\]
Since for $\pi \in \Pi(\mu,\nu)$, the application $P \mapsto  \int \|P(x-y)\|^2 \,d\pi(x,y)$ is continuous, the application $g: P \mapsto \Wd^2\left( P_\#\mu , P_\#\nu \right)$ is upper semicontinuous as the minimum of continuous functions. As the application $h: E \mapsto P_E$ is continuous, and $x \mapsto \sqrt{x}$ is nondecreasing, $f = \sqrt{g \circ h}$ is upper semicontinuous.

\paragraph{Proof of Lemma~\ref{lemma:renyi-cramer-wold}.}
Let $j \in J$. Since ${P_{E_j}}_\#\mu = {P_{E_j}}_\#\nu$, their characteristic functions are equal, \emph{i.e.} for all $t \in \Rd$,
\[
\begin{array}{lcl}
    \int \exp{i \langle t | x \rangle} \,d{P_{E_j}}_\#\mu(x)
    &=&
    \int \exp{i \langle t | x \rangle} \,d{P_{E_j}}_\#\nu(x) \\
    \int \exp{i \langle t | P_{E_j} x \rangle} \,d\mu(x)
    &=&
    \int \exp{i \langle t | P_{E_j} x \rangle} \,d\nu(x) \\
    \int \exp{i \langle P_{E_j} t | x \rangle} \,d\mu(x)
    &=&
    \int \exp{i \langle P_{E_j} t | x \rangle} \,d\nu(x)
\end{array}
\]
\emph{i.e.} the characteristic functions of $\mu$ and $\nu$ coincide on $E_j$, for all $j \in J$. Since the subspaces $\left(E_j\right)_{j \in J}$ cover the whole space $\Rd$, $\mu$ and $\nu$ have the same characteristic functions on $\Rd$, hence $\mu = \nu$.


\paragraph{Proof of the value of $\Wd^2(\mu,\nu)$ for the Disk to Annulus setup.}
Let us define a map $T$ using polar coordinates for the first two coordinates, and cartesian coordinates for the remaining $d-2$, as follows:
\[
	T\left(r, \theta, x_3, ..., x_d\right) = \left( \sqrt{4 + 5r^2}, \theta, x_3, ..., x_d \right).
\]
We show that $T$ is an optimal transport map between $\mu$ and $\nu$. First, we show that $T_\#\mu = \nu$. Since $T$ only operates on the first coordinate and $\mu$ and $\nu$ only differ on the first coordinate, we only have to prove that ${T_1}_\#\mu_1$ and $\nu_1$ have the same CDF, where $T_1$, $\mu_1$ and $\nu_1$ stand for the first coordinate projection of $T$, $\mu$ and $\nu$. For any $r \in [2,3]$:
\[
	\begin{split}
		\mathbb{P}_{R\sim\mu_1}(T_1(R) \leq r) &= \mathbb{P}_{R\sim\mu_1}(R \leq T_1^{-1}(r)) \\
		&= \mathbb{P}_{R\sim\mu_1}\left(R \leq \sqrt{\frac{r^2-4}{5}}\right) \\
		&= \frac{\int_0^{\sqrt{\frac{r^2-4}{5}}} x \,dx}{\int_0^1 x \,dx} = \frac{r^2-4}{5}
	\end{split}
\]
and
\[
	\mathbb{P}_{R\sim\nu_1}(R \leq r) = \frac{\int_2^r x \,dx}{\int_2^3 x \,dx} = \frac{r^2-4}{5}
\]
which shows that $T_\#\mu = \nu$. Moreover, $T$ is a subgradient of a convex function, since its gradient is semidefinite positive:
\[
	\nabla T(r,\theta,x_3,...,x_d) = \textrm{Diag}\left( \frac{5r}{\sqrt{4 + 5r^2}} , 1, ..., 1 \right) \succeq 0.
\]
Then by Brenier's theorem, $T$ is an optimal transport map between $\mu$ and $\nu$, and
\[
	\begin{split}
		\Wd^2(\mu, \nu) &= \int \|x-T(x)\|^2 \,d\mu(x) \\
		&= 2 \int_0^1 \left(r - \sqrt{4+5r^2}\right)^2 r \,dr \\
		&= \frac{14}{5} + \frac{8}{5 \sqrt{5}} \log\left( \frac{3 + \sqrt{5}}{2} \right) \approx 3.48865.
	\end{split}
\]

\newpage
\section{Experimental Details}

\subsection{Additional Experiment: Transport from Disk to Annulus}
Let $k^* \in \range{d}$. We now consider $\mu$ the uniform distribution over the $k^*$-dimensional disk embedded in $\Rd$,
\[
	\begin{split}
		\mu = \mathcal{U}(\{ x \in \Rd \,|\, &\|(x_1,...,x_{k^*})\| \leq 1, \\
		&x_i \in [0,1] \textrm{ for } i=(k^*+1),...,d \})
	\end{split}
\]
and $\nu$ the uniform distribution over a $k^*$-dimensional annulus (cylinder) embedded in $\Rd$,
\[
	\begin{split}
		\nu = \mathcal{U}(\{ x \in \Rd \,|\, & 2 \leq \|(x_1,...,x_{k^*})\| \leq 3, \\
		&x_i \in [0,1] \textrm{ for } i=(k^*+1),...,d \}).
	\end{split}
\]

We do the same experiments as for the fragmented hypercube. Based on two empirical distributions $\hat\mu$ from $\mu$ and $\hat\nu$ from $\nu$, we plot in Figure~\ref{fig:disktoannulus_curve_k} the sequence $k \mapsto \ProjWasserstein{k}^2(\hat\mu, \hat\nu)$, for different values of $k^*$. An ``elbow'' shows at $k=k^*$, because the last $d-k^*$ dimensions only represent noise, which is recovered in our plot.

\begin{figure}[h!]
	\centering
	\includegraphics[width=0.48\textwidth]{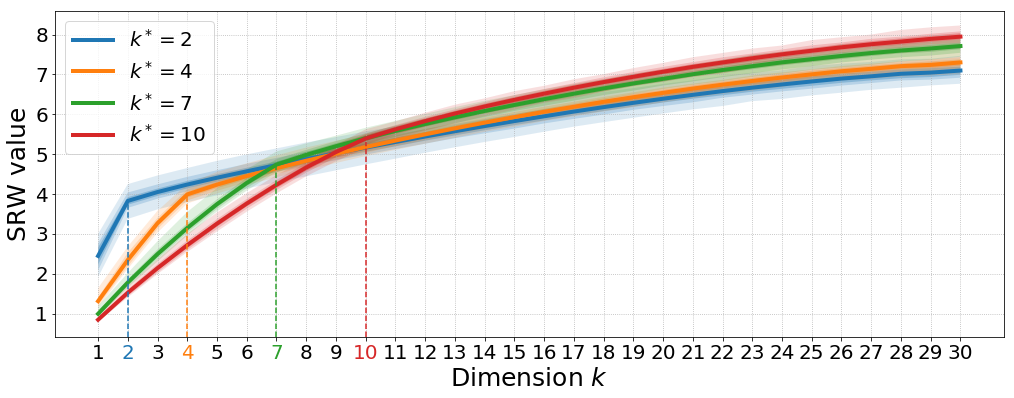}
 	\vskip-.4cm  
 	\caption{$\ProjWasserstein{k}^2(\hat\mu, \hat\nu)$ depending on the dimension $k \in \range{d}$, for $k^* \in \{2, 4, 7, 10\}$, where $\hat\mu, \hat\nu$ are empirical measures from $\mu$ and $\nu$ respectively with $100$ points each. Each curve is the mean over $100$ samples, and shaded area show the minimum and maximum values.}
 	\label{fig:disktoannulus_curve_k}
\end{figure}

We consider next $k^*=2$, and choose $k=2$. We will need to compute $\Wd^2(\mu, \nu)$. Although~\cite{weed19} seem to suggest that it is equal to $4$, we find a different value :
\[
	\Wd^2(\mu, \nu) = \frac{14}{5} + \frac{8}{5 \sqrt{5}} \log\left( \frac{3 + \sqrt{5}}{2} \right) \approx 3.48865.
\]
We plot in Figure~\ref{fig:disktoannulus_estimation_error} the estimation error $|\Wd^2(\mu, \nu) - \ProjWasserstein{k}^2(\hat\mu, \hat\nu)|$ depending on the number of points $n$ in the empirical measures $\hat\mu, \hat\nu$ from $\mu$ and $\nu$.  In Figure~\ref{fig:disktoannulus_subspace_estimation_error}, we plot the subspace estimation error $\|\Omega^* - \widehat\Omega\|$ depending on $n$, where $\Omega^*$ is the optimal projection matrix onto $\textrm{span}\{e_1, e_2\}$.
\begin{figure}[h!]
	\centering
	\includegraphics[width=0.48\textwidth]{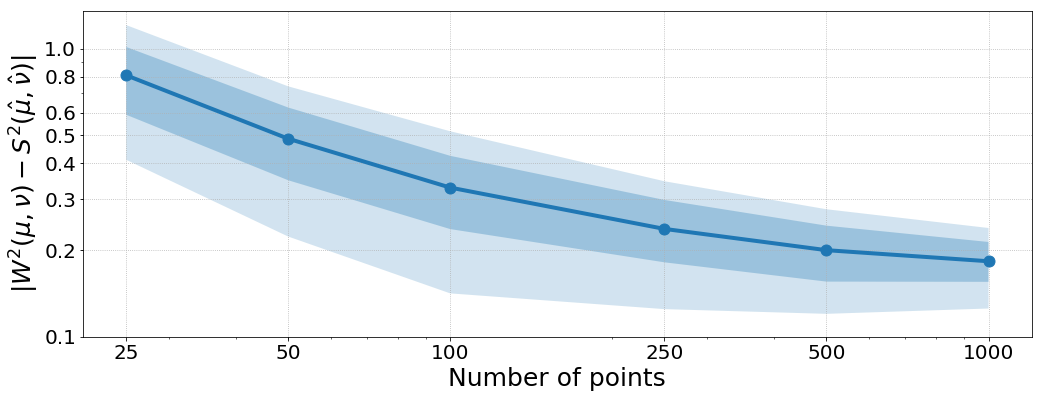}
	\vskip-.5cm
	\caption{Mean estimation error over $500$ random samples for $n \in \{25, 50, 100, 250, 500, 1000\}$. The shaded areas represent the 10\%-90\% and 25\%-75\% quantiles over the $500$ samples.}
    \label{fig:disktoannulus_estimation_error}
\end{figure}

\begin{figure}[h!]
	\centering
	\includegraphics[width=0.48\textwidth]{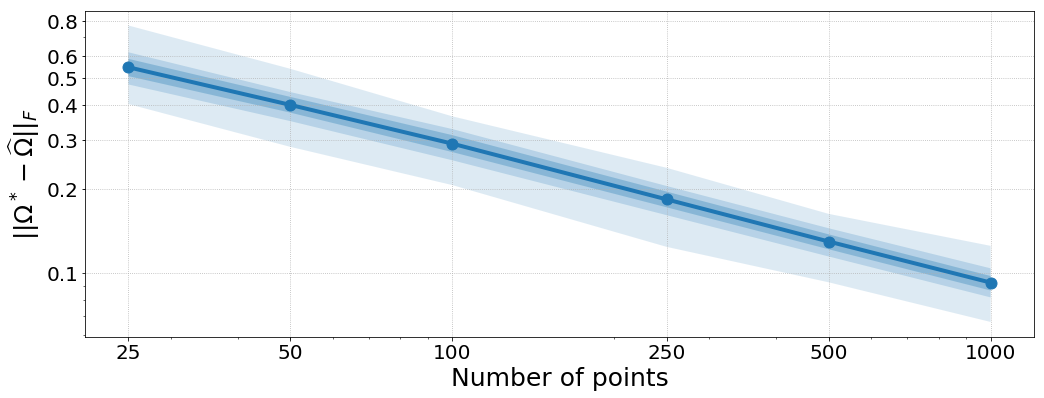}
	\vskip-.5cm
	\caption{Mean estimation of the subspace estimation error over $500$ samples, depending on $n \in \{25, 50, 100, 250, 500, 1000\}$. The shaded areas represent the 10\%-90\% and 25\%-75\% quantiles over the $500$ samples.}
	\label{fig:disktoannulus_subspace_estimation_error}
\end{figure}

We plot the optimal transport plan (in the sense of $\Wd$, Figure~\ref{fig:disktoannulus_map} left) and the optimal transport plan (in the sense of $\ProjWasserstein{2}$) between $\hat\mu$ and $\hat\nu$ (with $n=250$ points each, Figure~\ref{fig:disktoannulus_map} right).
\begin{figure}[h!]
	\centering
	\begin{subfigure}
       		 \centering
		\includegraphics[width=0.22\textwidth]{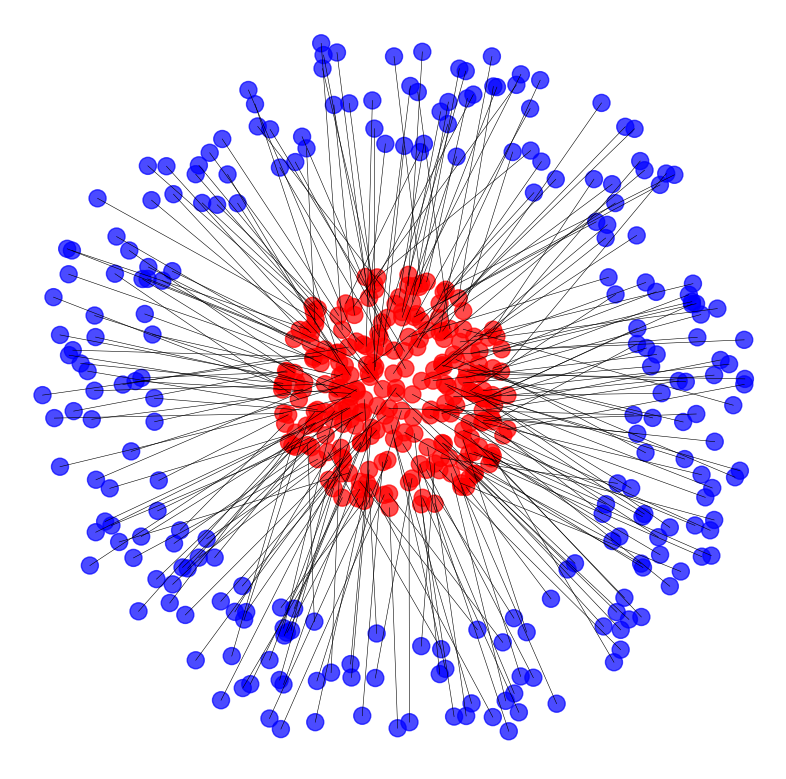}
    	\end{subfigure}%
    	~ 
   	 \begin{subfigure}
      		\centering
        		\includegraphics[width=0.22\textwidth]{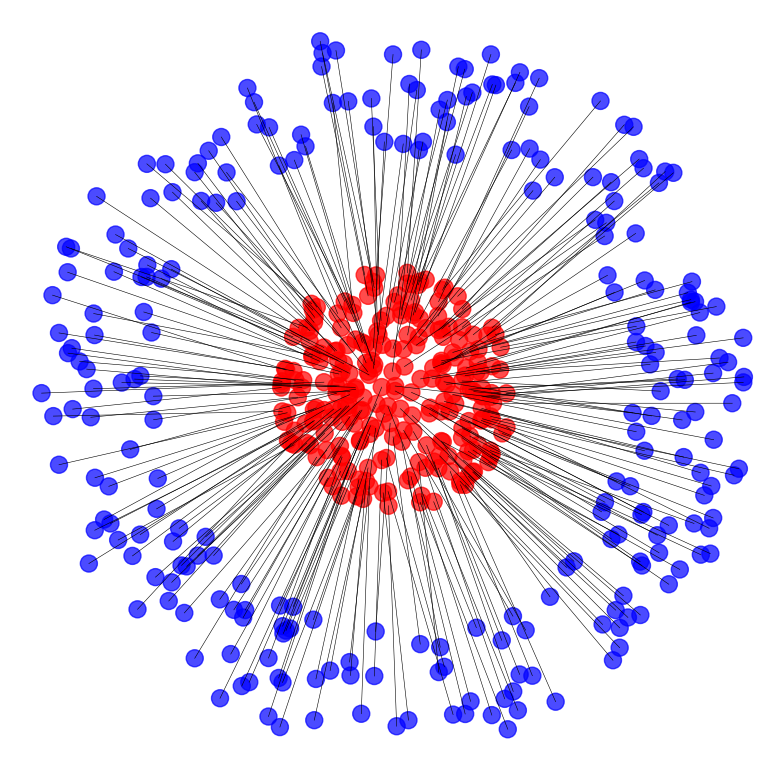}
   	 \end{subfigure}
\vskip-.5cm	 \caption{Disk to annulus, $n=250$, $d=30$. Optimal mapping in the Wasserstein space (left) and in the SRW space (right). Geodesics in the SRW space are robust to statistical noise.}
		\label{fig:disktoannulus_map}
\end{figure}

\subsection{Details about Experiment of Section \ref{exp:real}}
The complete vocabulary used consists of the $20000$ most common words in English, except for the $2000$ most common words, hence a total size of $18000$ words. All the words in a movie script that whether do not belong to the vocabulary list, are digits or begin with a capital letter, are deleted. The remaining words form a discrete measure in $\R^{300}$, with the weights proportional to their frequency in the movie script.

\end{document}